\documentclass[10pt, conference, letterpaper]{IEEEtran}

\usepackage{lmodern}

\usepackage[utf8]{inputenc} 
\usepackage[T1]{fontenc}    
\usepackage{hyperref}       
\usepackage{url}            
\usepackage{booktabs}       
\usepackage{amsfonts}       
\usepackage{nicefrac}       
\usepackage{microtype}      

\usepackage[caption=false,font=footnotesize]{subfig}
\usepackage{url}
\usepackage[USenglish,UKenglish,english]{babel} 

\makeatletter
\newcommand{\upquotetype}{}
\newcommand{\upquote@aux}[1]{\text{\upquotetype}#1\text{\upquotetype}}
\newcommand{\upquotesingle}{\renewcommand{\upquotetype}{\textquotesingle}\upquote@aux}
\newcommand{\upquotedouble}{\renewcommand{\upquotetype}{\textquotedbl}\upquote@aux}

\usepackage{textcomp}

\usepackage{graphicx} 
\usepackage{epstopdf}
\usepackage{amsthm}
\newtheorem{theorem}{Theorem}
\usepackage{amsmath}
\usepackage{amssymb}
\newcommand{\argmin}{\operatornamewithlimits{argmin}}
\newcommand{\argmax}{\operatornamewithlimits{argmax}}

\usepackage{algorithm}
\usepackage{algorithmic}

\title{Large Scale Behavioral Analytics via Topical Interaction}

\author{
  Shih-Chieh Su\thanks{Connect with the author: \url{https://www.linkedin.com/in/jessysu} } \\
  Information Security and Risk Management Department\\
  Qualcomm Inc.\\
  San Diego, CA, 92121 \\
  \texttt{shihchie@qualcomm.com} \\
}

\begin{document}

\maketitle

\begin{abstract}
We propose the split-diffuse (SD) algorithm that takes the output of an existing dimension reduction algorithm, and distributes the data points uniformly across the visualization space. The result, called the topic grids, is a set of grids on various topics which are generated from the free-form text content of any domain of interest. The topic grids efficiently utilizes the visualization space to provide visual summaries for massive data. Topical analysis, comparison and interaction can be performed on the topic grids in a more perceivable way. 
\end{abstract}

\section{Introduction}\label{sec-intro}
When there are multiple measures of the each sample, the data is described in the a high dimensional space $\mathcal{H}$ by these measures. For example, the samples collected by a network of sensors have the dimensionality of the number of sensors. Each stock symbol can be described by multiple quantifiers which make up the data space. Documents of a corpus can be described by the vector of word counts (``word vector of features''), which can be all or most of the words appearing in the corpus. On these cases, it is possible that $\mathcal{H}$ has tens of thousands dimensions.

To make these high dimensional data points visible to human, a word embedding (or dimension reduction) technique is employed to map the data points to a lower dimensional space $\mathcal{L}$. Usually $\mathcal{L}$ is a two-dimensional (2D) or three-dimensional (3D) space. The word embedding technique attempts to preserve some relationship among the data points in $\mathcal{H}$ after mapping them to $\mathcal{L}$. 

For example, the principal component analysis (PCA) \cite{Pearson01} attempts to combine the dimensions of $\mathcal{H}$ to form the dimensions in $\mathcal{L}$ that best explain the variance in the data. However, the directions maximizing the variance do not always maximize the information, or optimize some other desirable metric. In some case, one would prefer the inter-point relationship in $\mathcal{H}$ being maintained in $\mathcal{L}$. The multi-dimensional scaling (MDS) \cite{Torgerson52} $\mathbb{M}$ tries to preserve the distance between data points $\{p\}$, during the mapping from $\mathcal{H}$ to $\mathcal{L}$, so that
\begin{equation}
\mathbb{M} = \argmin_{\forall M: p_{(\mathcal{H})} \rightarrow p_{(\mathcal{L})}} \sum_{\forall i,j} \left( \frac{||p_i-p_j||_{(\mathcal{H})} - ||p_i-p_j||_{(\mathcal{L})}} {||p_i-p_j||_{(\mathcal{H})}} \right) ^ 2
\end{equation}

The stochastic neighbor embedding (SNE) \cite{Hinton02} type of algorithms further emphasize the local distances. In a broader sense, SNE algorithms try to preserve the local relationship ahead of preserving the global relationship. There are other techniques emphasizing different favored metrics over the relationship among data points. Some popular choices include isomap \cite{Tenenbaum00}, spectral embedding \cite{Belkin03}, and totally random trees embedding \cite{Geurts06}). On a specific situation, one particular dimension reduction technique could be more suitable than others. 

In this paper, we do not discuss the use cases for different dimension reduction algorithms. Instead, we focus on the techniques in the 2D or 3D visual space to help human perceive and interact with the rendered information more easily.

\begin{figure*}[htp]
  \centering
  \subfloat[first-level split]{\includegraphics[scale=0.19]{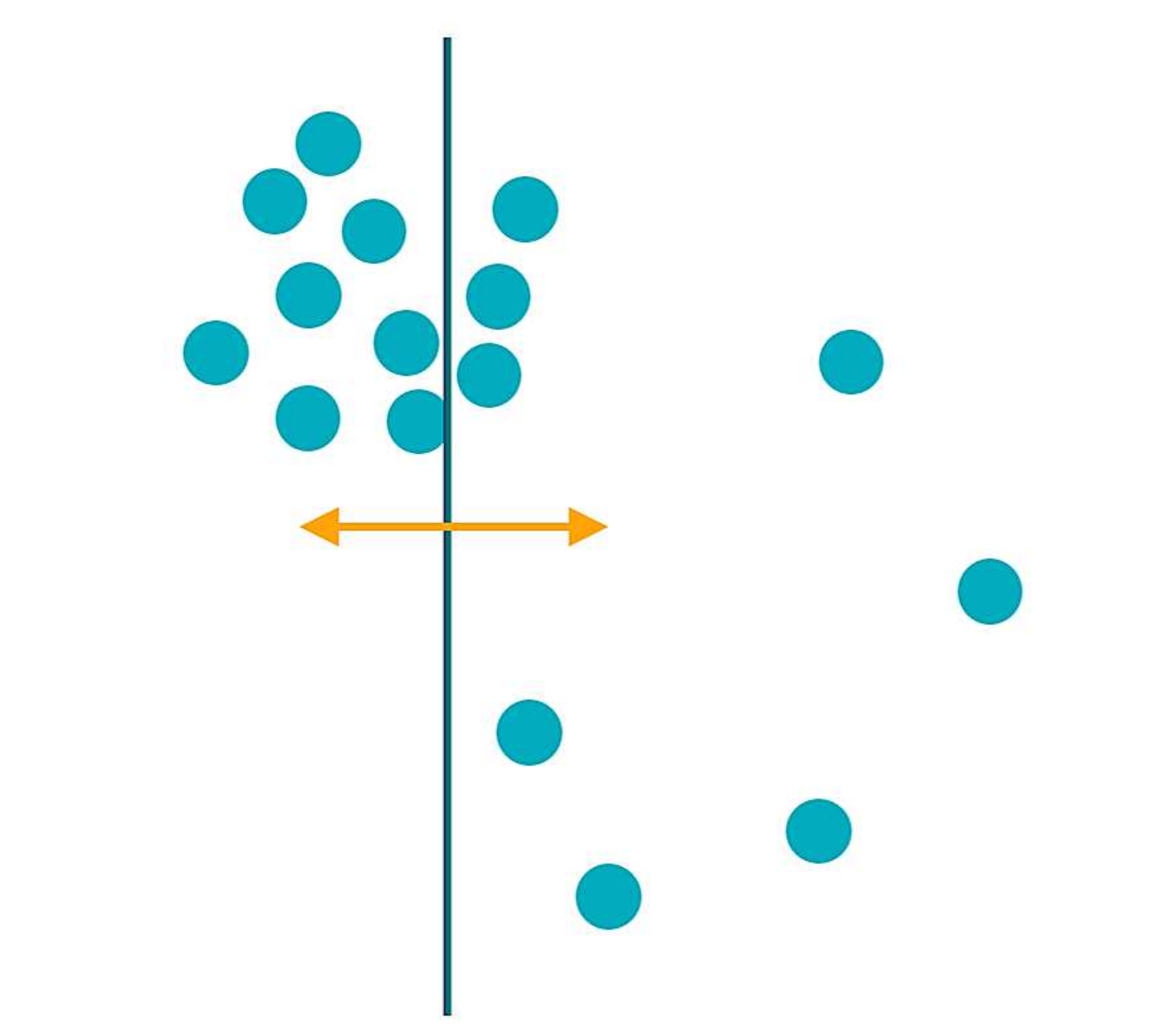}} \quad
  \subfloat[second-level splits]{\includegraphics[scale=0.19]{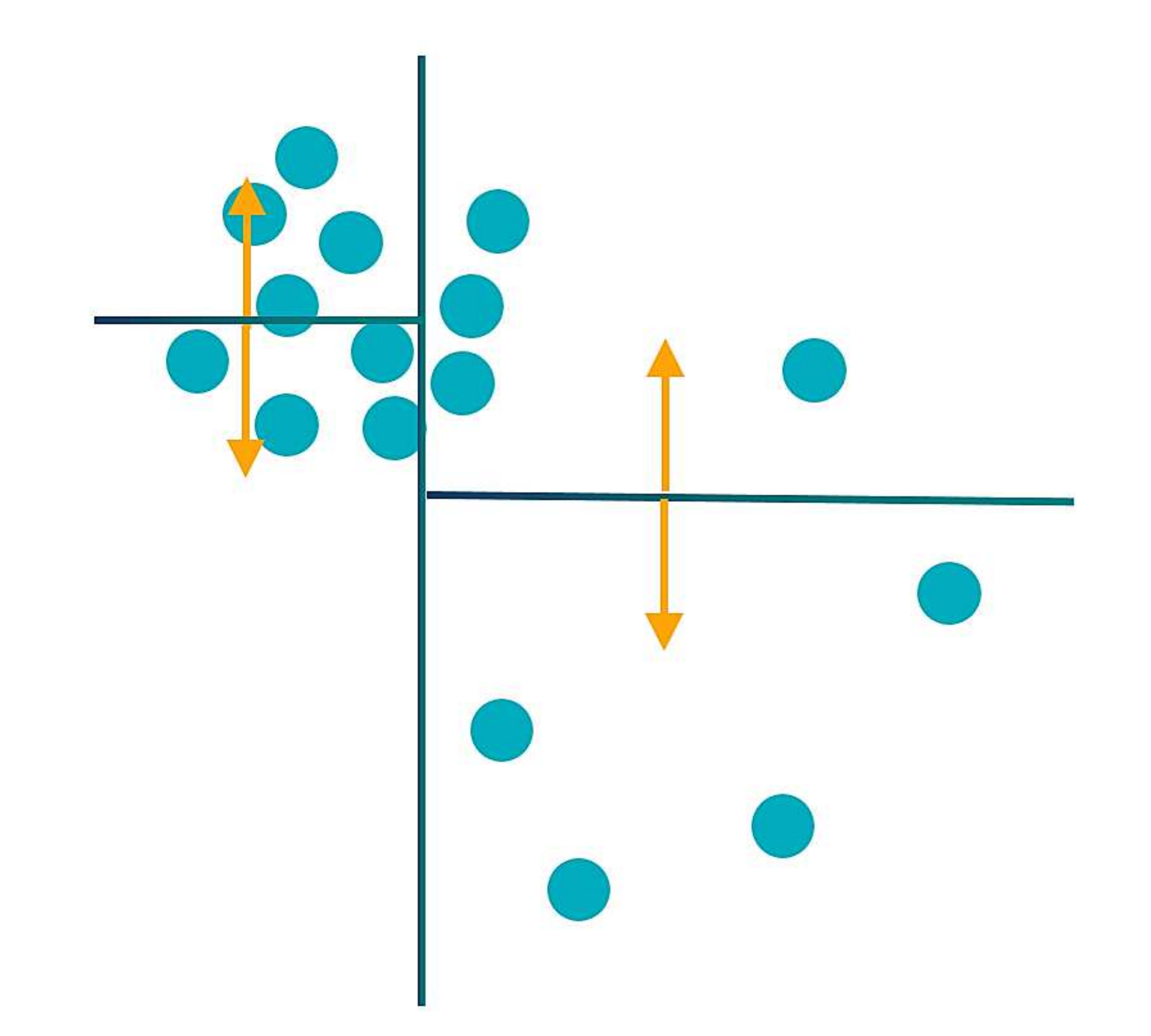}} \quad
  \subfloat[sample splitting path]{\includegraphics[scale=0.19]{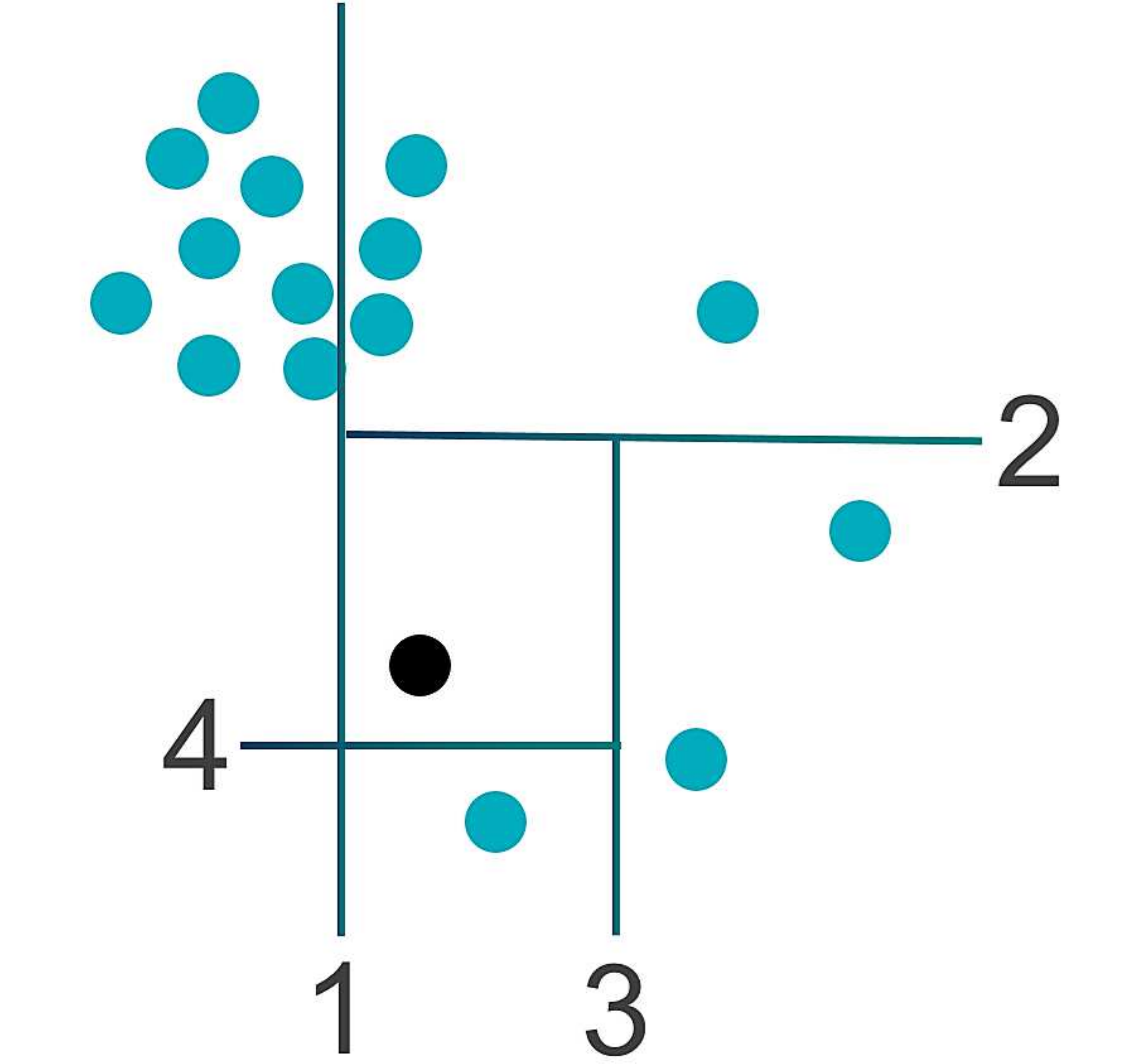}}
\caption{The split-diffuse algorithm. (a) and (b) show splits on a set of sample points over the $4\times4$ layout. (c) illustrates the path leading to the final placemant of the block point.}
\label{fig-SD}
\end{figure*}

\section{Background}\label{sec-background}

Consider the case to visualize the traffic drift of a network node, over two periods of time. There could be hundreds of million entries in the activity log during each time period. Rather than showing these entries line-by-line, a common approach is to classify the entries into topics, based on the entry content. Then we can compare the topical metrics across different time periods. Some example metrics can be entry counts, total file size, total packet size, counts in source or destination entities.

Similarly, one may be interested in comparing the shipping patterns on millions of shipments to the New York City against that to the Los Angeles City. Or compare the reimbursement patterns before and after a policy change of a healthcare insurance plan. In these cases, showing the massive amount of entries directly in the charts can be overwhelming and difficult to make comparison.

Instead, we want a visual summary for the behavioral content. The same summarization and visualization method can be performed on different metrics, on different targets, or over different period of time. Thus the human expert can compare the behavior, and can interact with the differences.

We start from the log of activities. Each entry has the identifier of the entity that incurred the activity, the timestamp, and the activity detail. The activity detail can be any free-form text and is properly punctuated. We collect all these entries to form a corpus. Each entry (document) becomes a data point in the high dimensional word vector space $\mathcal{H}$ defined by the corpus. Topics are generated in $\mathcal{H}$ to summarize the activities, using techniques like the latent Dirichlet allocation (LDA) model \cite{Blei03}. These topics, in the form of data points, are than mapped to $\mathcal{L}$ with the existing dimension algorithm of choice.

The output from existing dimension reduction algorithm is a set of data points that are non-uniformly scattered around the visualization space. This helps to explain the clustering behavior, including inter-cluster and intra-cluster, among the data points. However, there are also some drawbacks:
\begin{enumerate}
  \item The clusters tangle with others; some data points overlap with others. Overlap makes the information less perceivable.
  \item The data points are denser in some area. The heterogeneity makes human interaction with the data points more difficult.
  \item When comparing different metrics over the same space, the geometrical relationship between the data points tends to outshine the comparison between the metrics.
\end{enumerate}

Moreover, there is always area where points are sparse. The white space that is not occupied by any points does not provide any extra information other than having no data point. We cannot interact with the white space the way we interact with the topic data points. In our case, we want to know the nearby topics of a specific topic, but the distance accuracy can be traded for the ease of topical comparison and topical interaction.

In order to better utilize the visualization space, we suggest to distribute the data points evenly over the visualization space. The cloud of data points is deformed in the same space $\mathcal{L}$ defined by the dimension reduction algorithm. This deformation is denoted as $\mathbb{S}$. In the meanwhile, it is desirable to preserve the point-wise relationship maintained by the dimension reduction algorithm. Our strategy in approaching this goal is prioritized as follows:
\begin{enumerate}
  \item Points are equally spaced.
  \item Point-wise topology is preserved.
  \begin{enumerate}
    \item When the point $p_i$ is to the right of point $p_j$ in space $\mathcal{L}$, $\mathbb{S}$ attempts to make $\mathbb{S}(p_i)$ to the right of $\mathbb{S}(p_j)$ in the same space $\mathcal{L}$, for all $i,j$, within all dimensions of $\mathcal{L}$.
  \end{enumerate}
  \item Point-wise geometry is loosely followed.
  \begin{enumerate}
    \item When $p_i$ is far from $p_j$, $\mathbb{S}(p_i)$ is far from $\mathbb{S}(p_j)$.
    \item When $\{p_i\}$ forms a cluster, $\{\mathbb{S}(p_i)\}$ still forms a cluster (the definition on cluster may vary over different situations).
  \end{enumerate}
\end{enumerate}

\begin{figure*}[htp]
  \centering
  \subfloat[2D MDS output]{\includegraphics[scale=0.22]{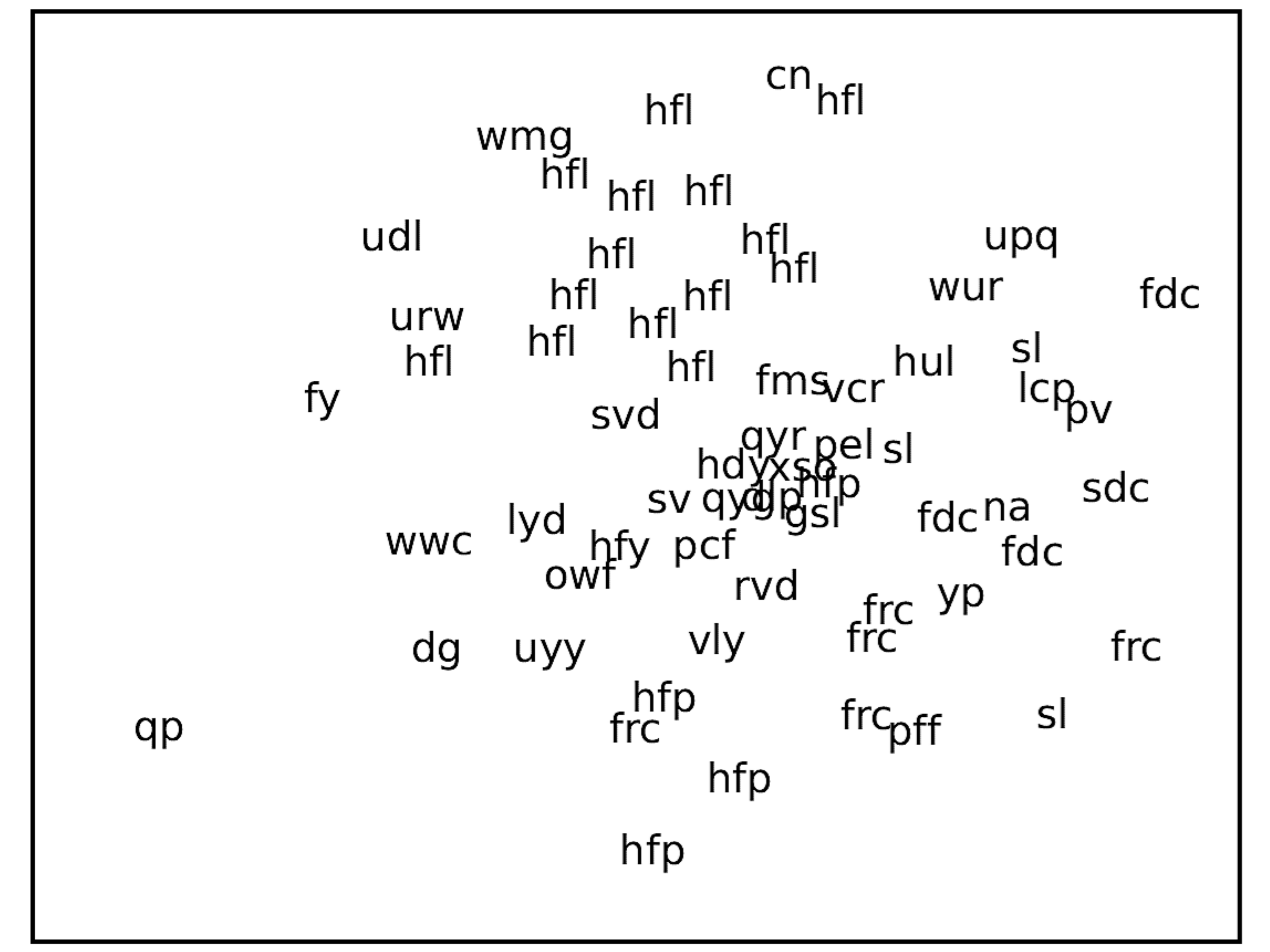}} 
  \subfloat[SD output from (a)]{\includegraphics[scale=0.22]{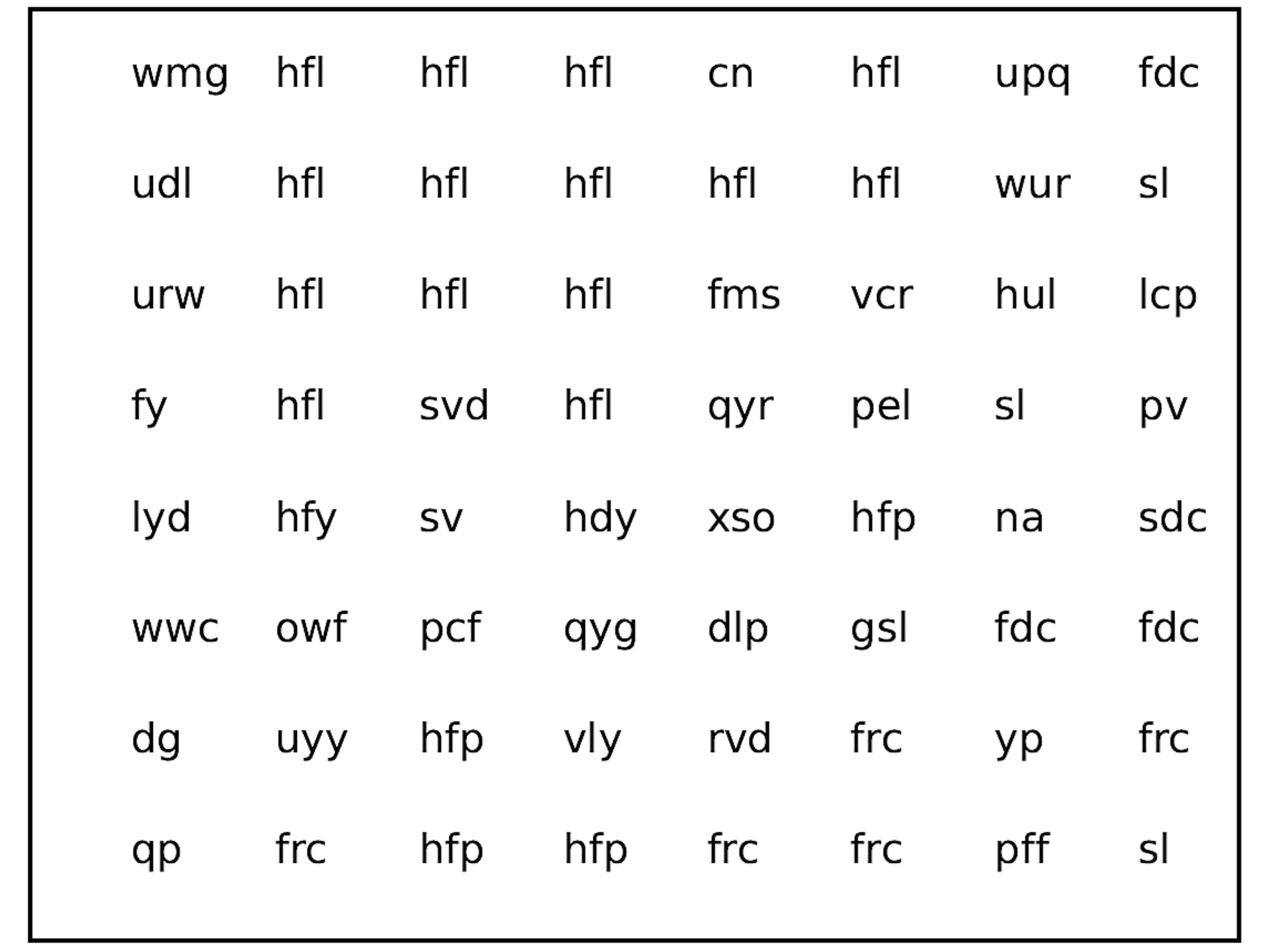}}
  \\
  \subfloat[2D t-SNE output]{\includegraphics[scale=0.22]{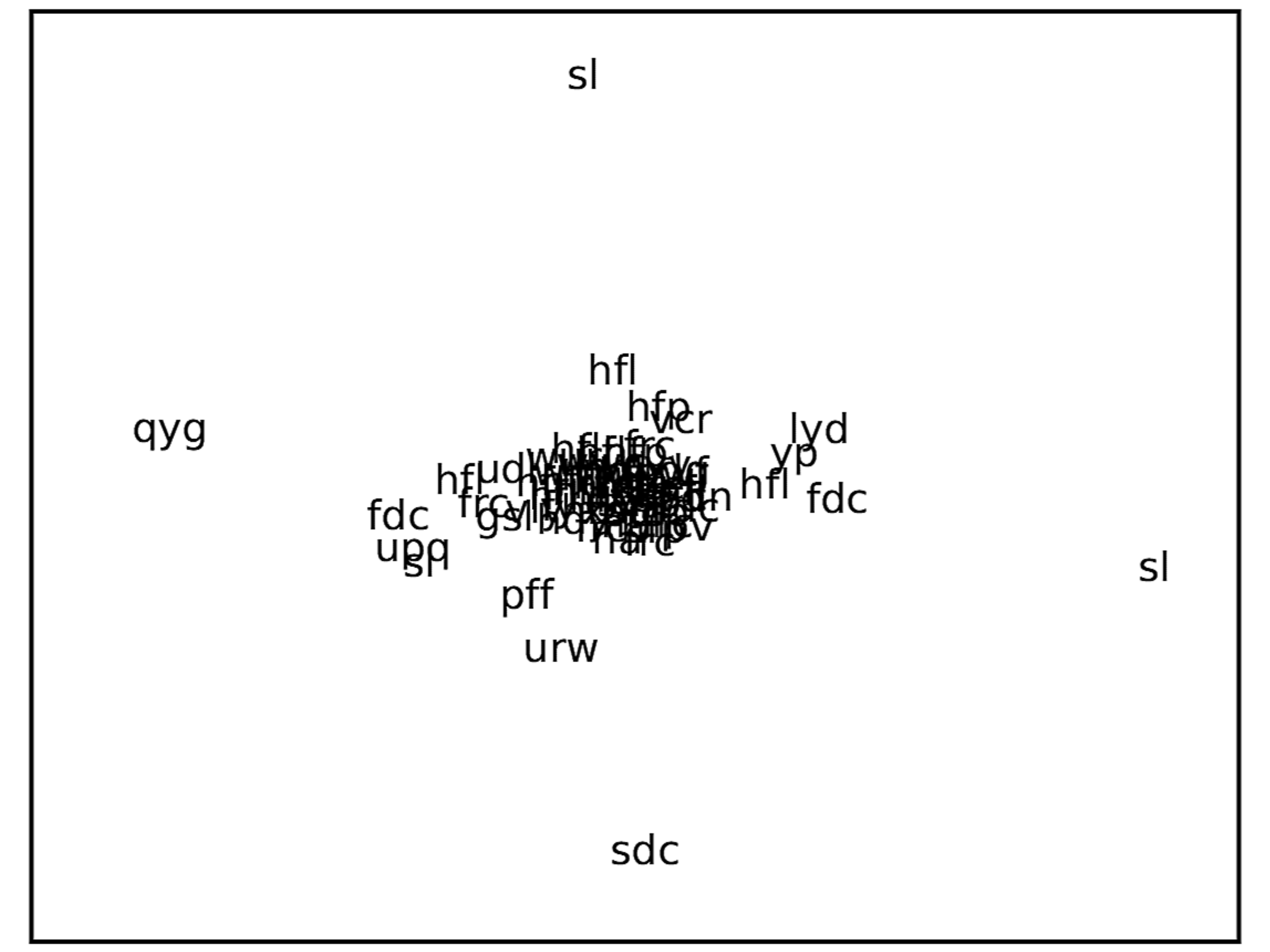}}
  \subfloat[SD output from (c)]{\includegraphics[scale=0.22]{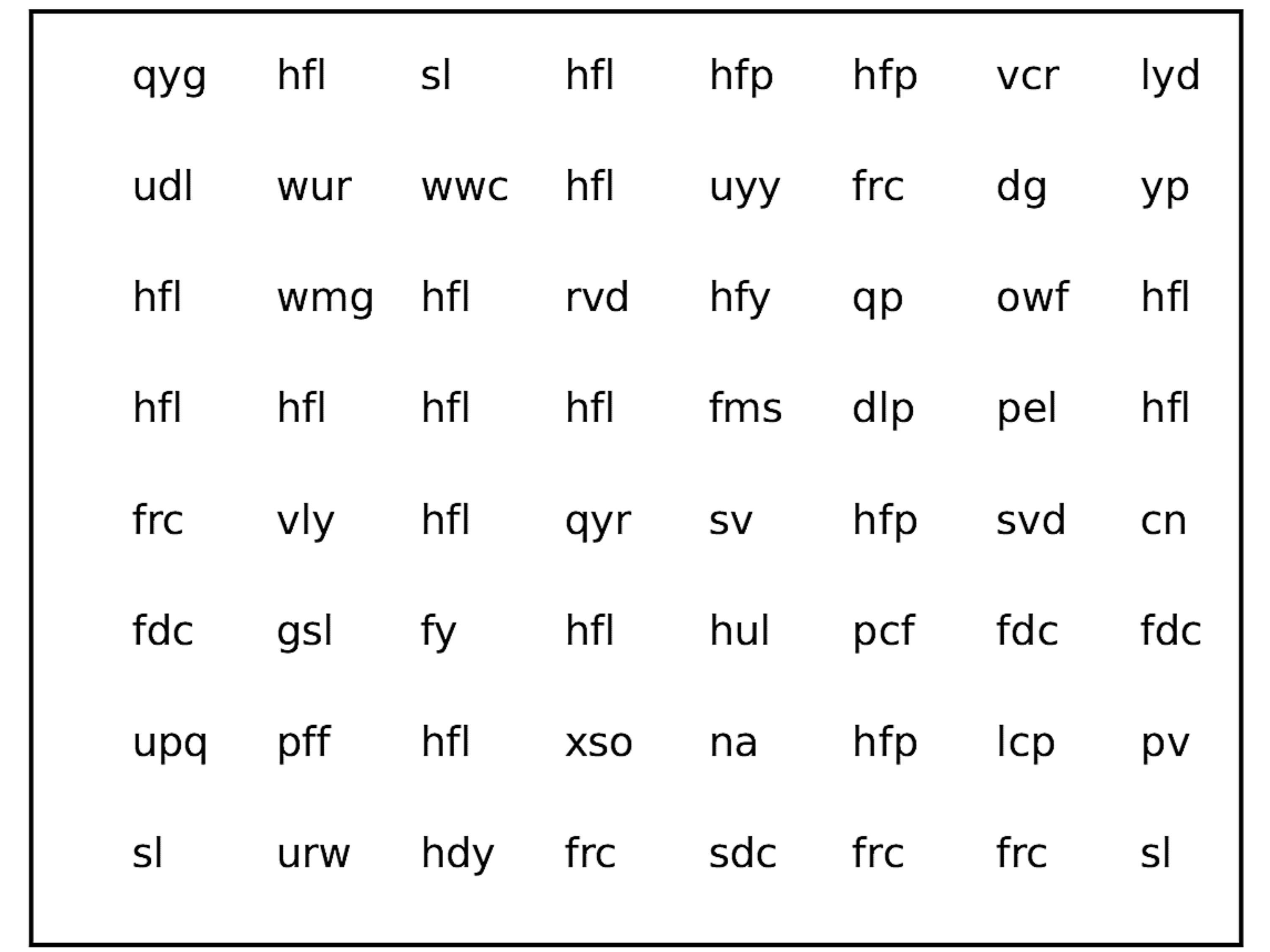}}
  \\
  \subfloat[3D MDS output]{\includegraphics[scale=0.22]{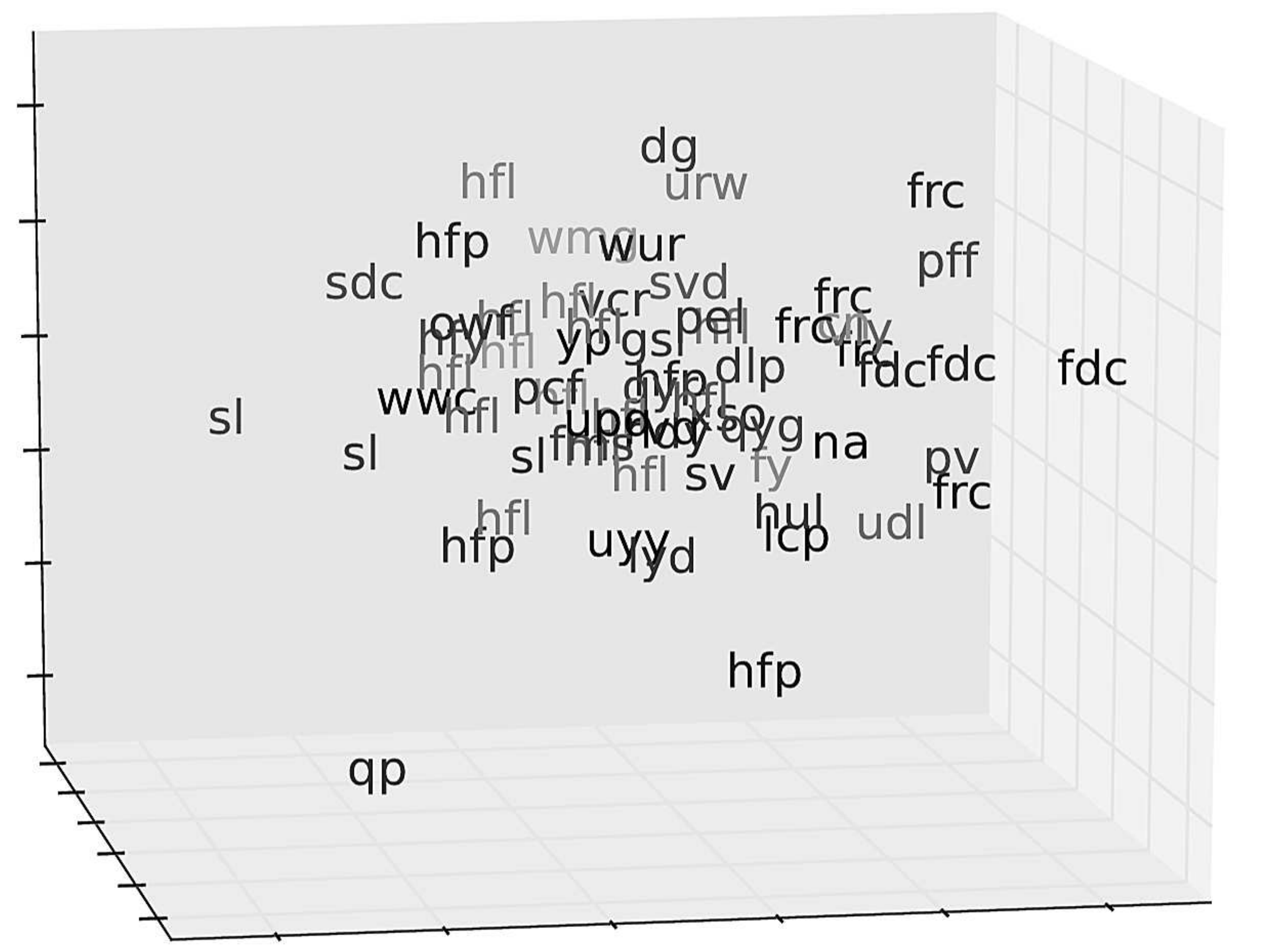}}
  \subfloat[SD output from (e)]{\includegraphics[scale=0.22]{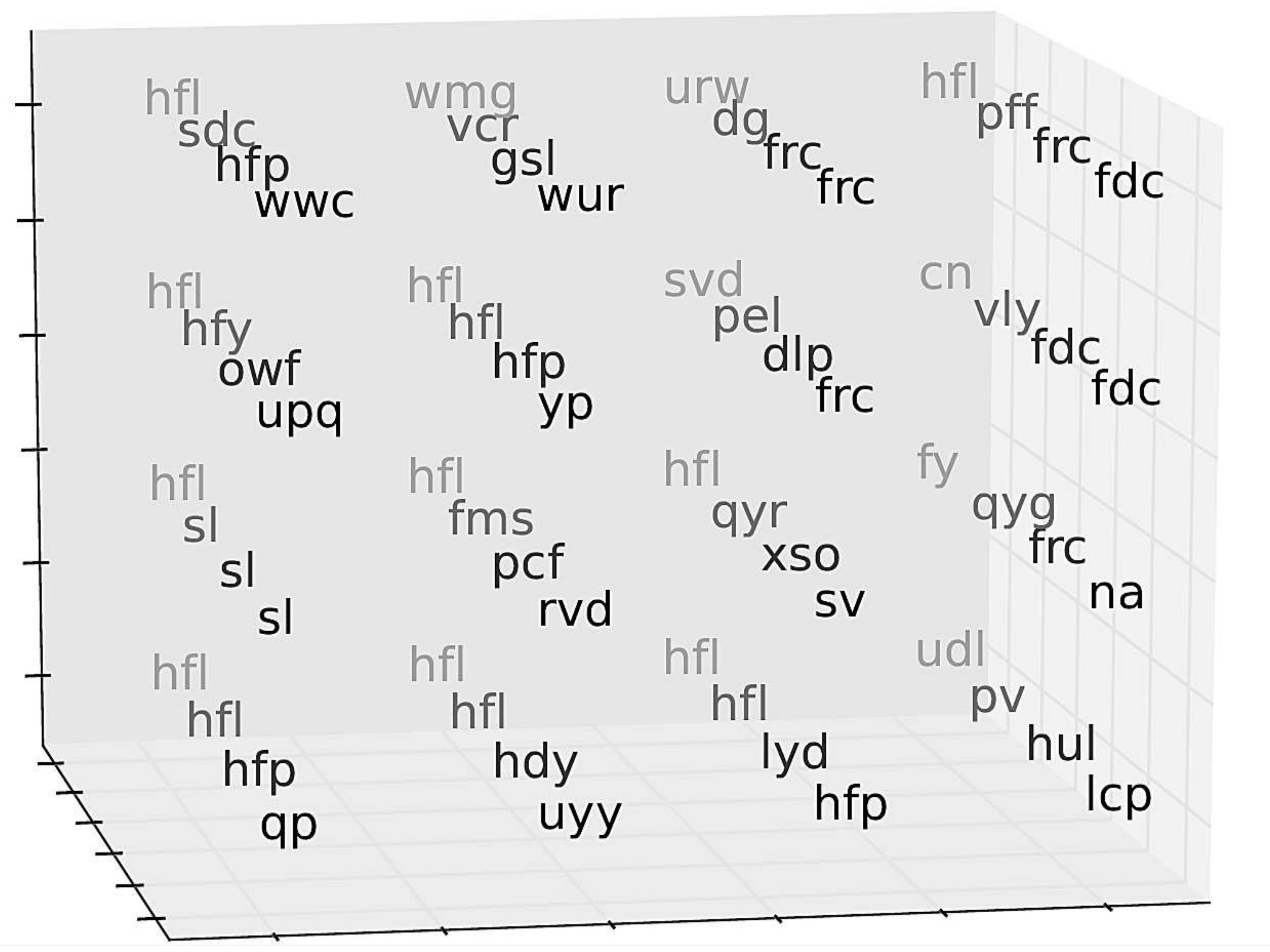}}
  \\
  \subfloat[3D t-SNE output]{\includegraphics[scale=0.22]{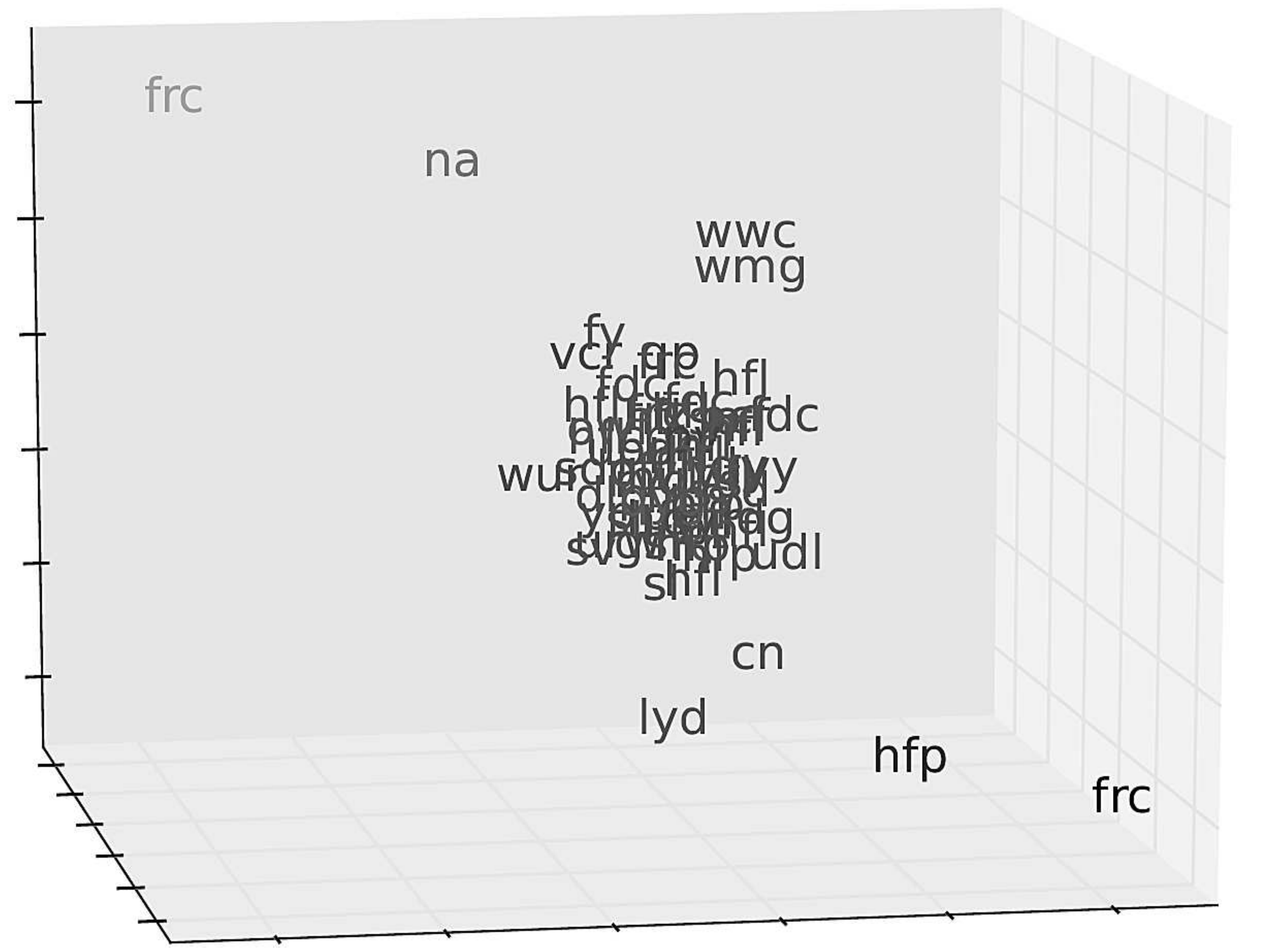}}
  \subfloat[SD output from (g)]{\includegraphics[scale=0.22]{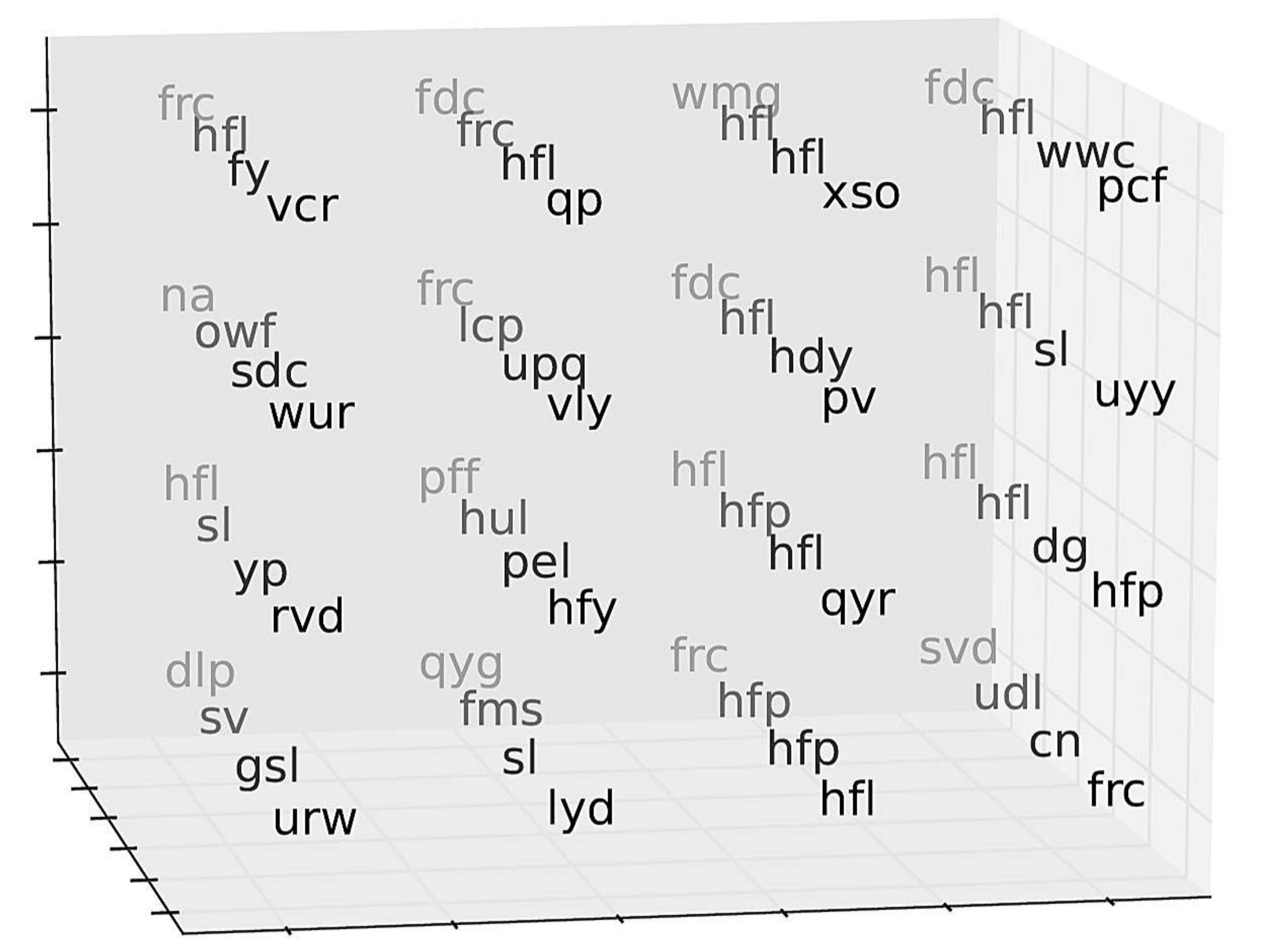}}
\caption{Examples of placing $64$ data points over different layouts with the SD algorithm.}
\label{fig-mdstsne}
\end{figure*}

\begin{figure*}[htp]
  \centering
  \subfloat[current activities of an entity]{\includegraphics[scale=0.16]{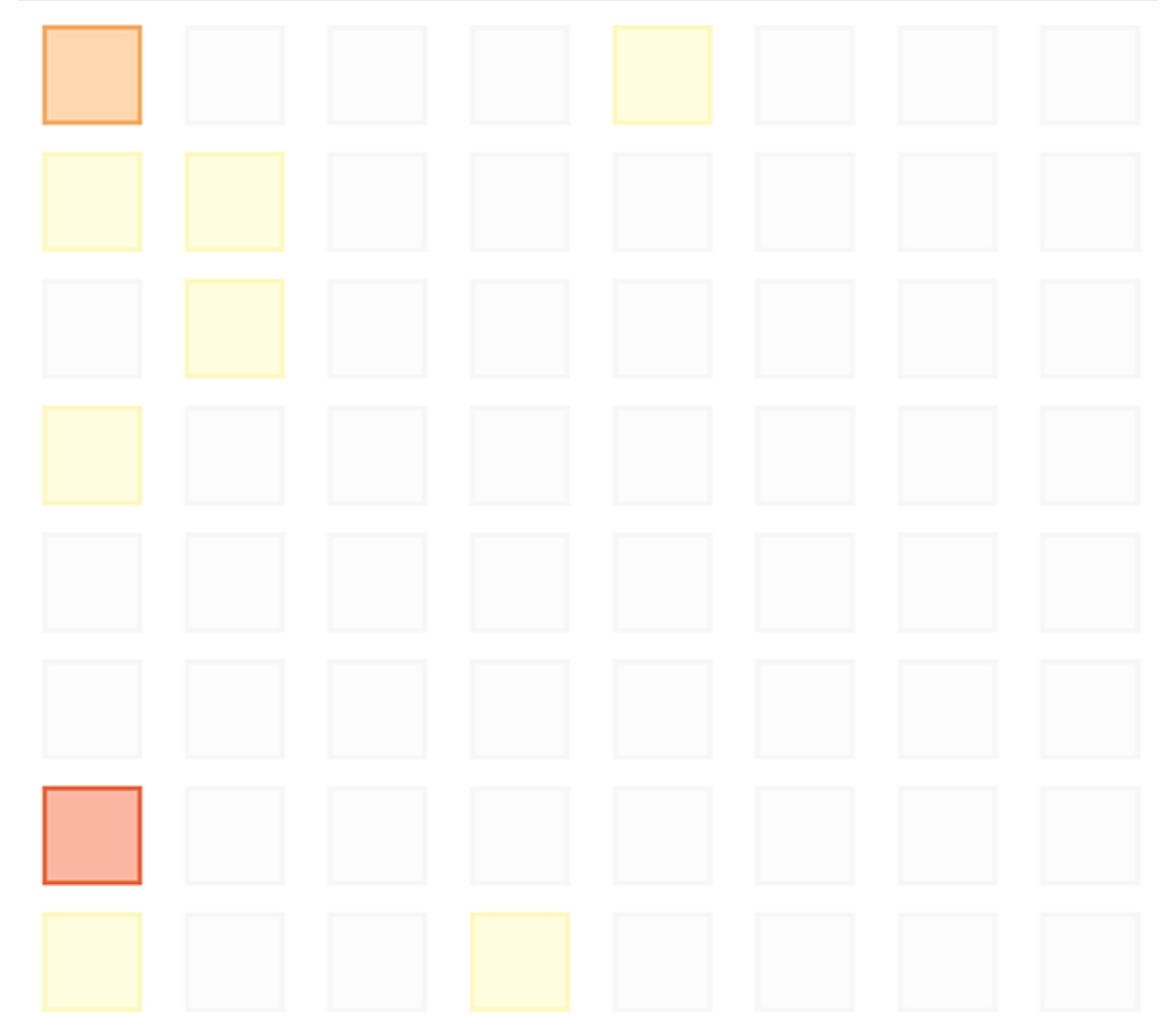}}\quad
  \subfloat[historical activities of this entity]{\includegraphics[scale=0.16]{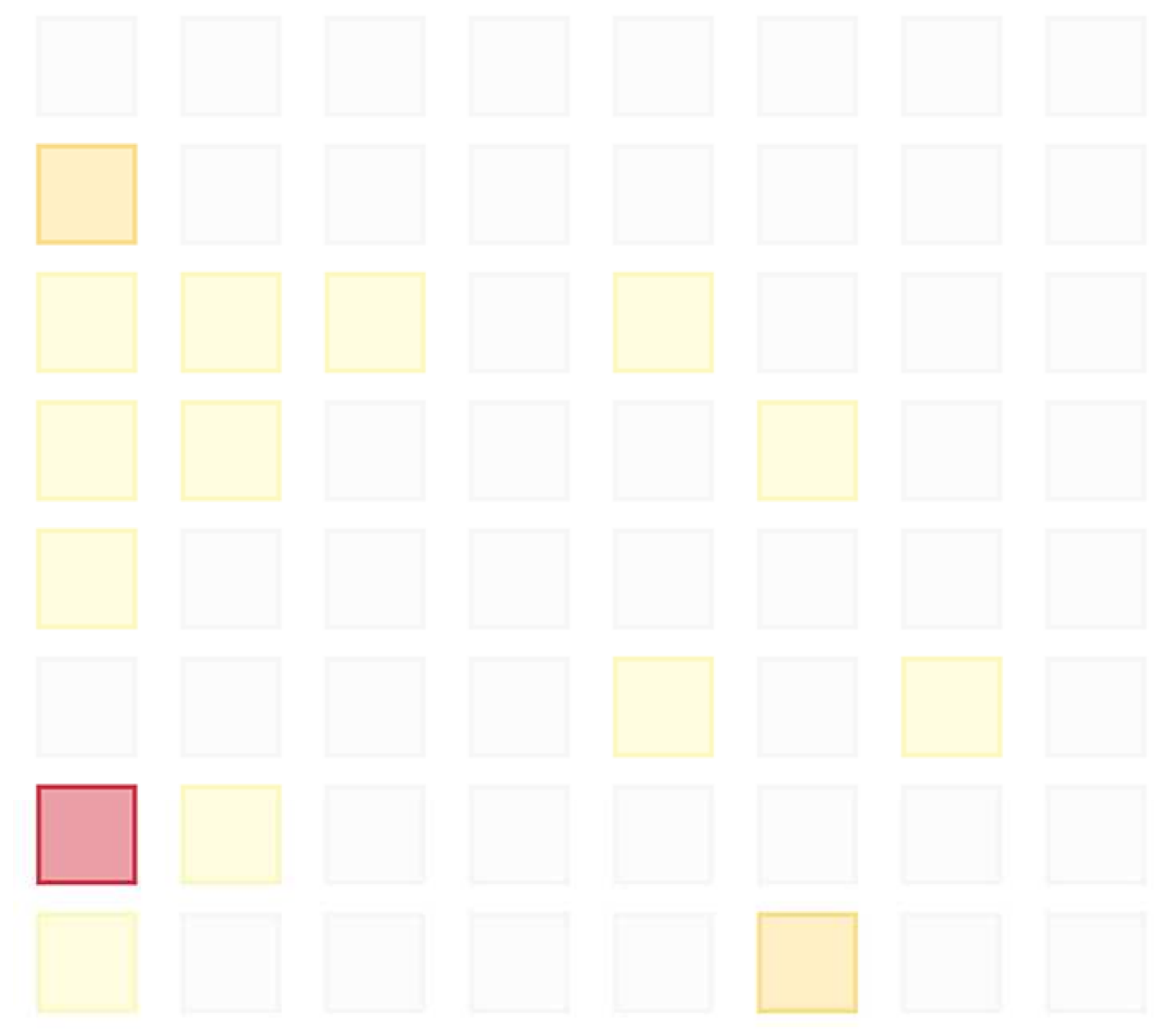}}\quad
  \subfloat[risk against the historical self]{\includegraphics[scale=0.16]{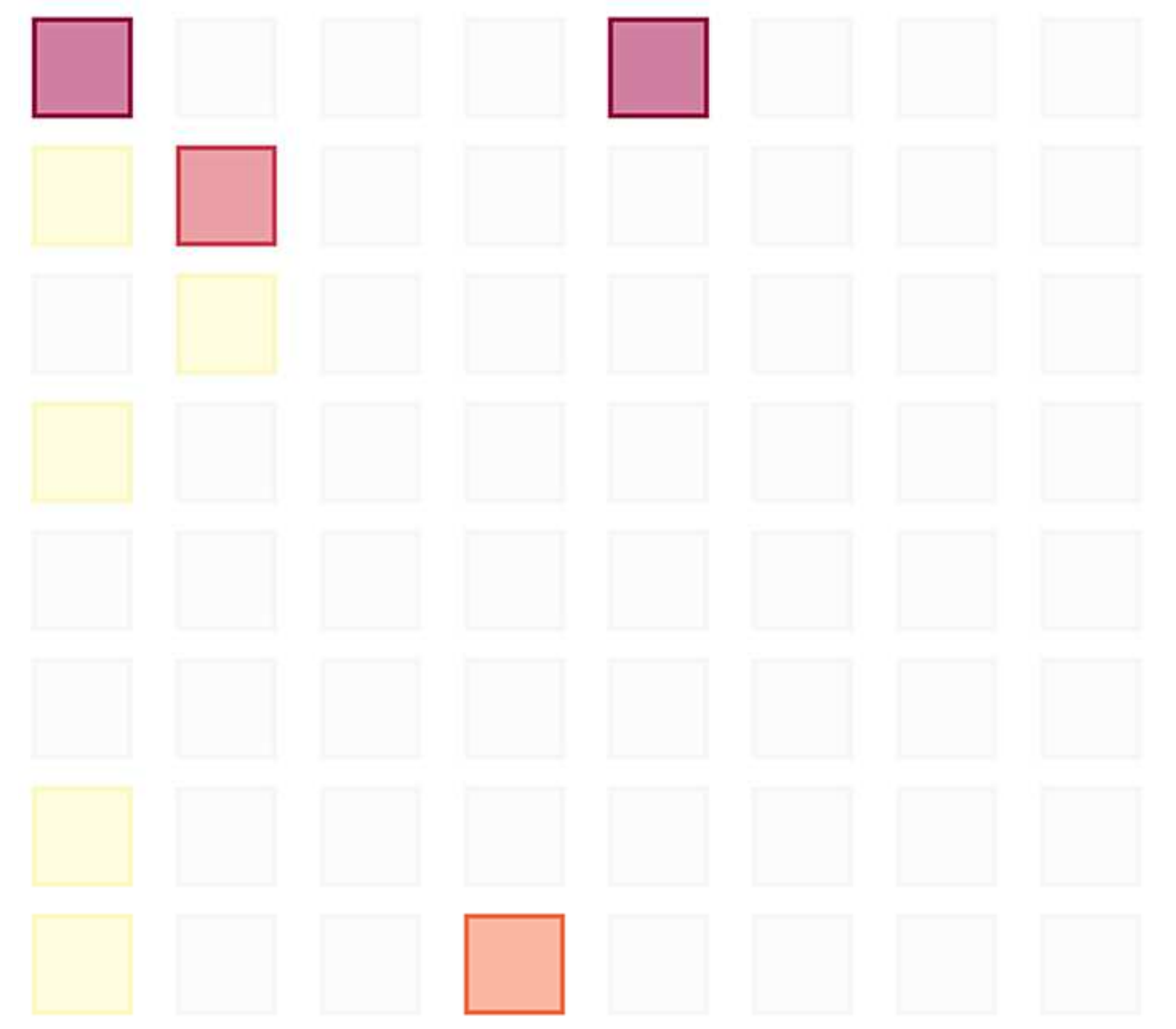}}\quad
  \subfloat[historical activities of the peers]{\includegraphics[scale=0.16]{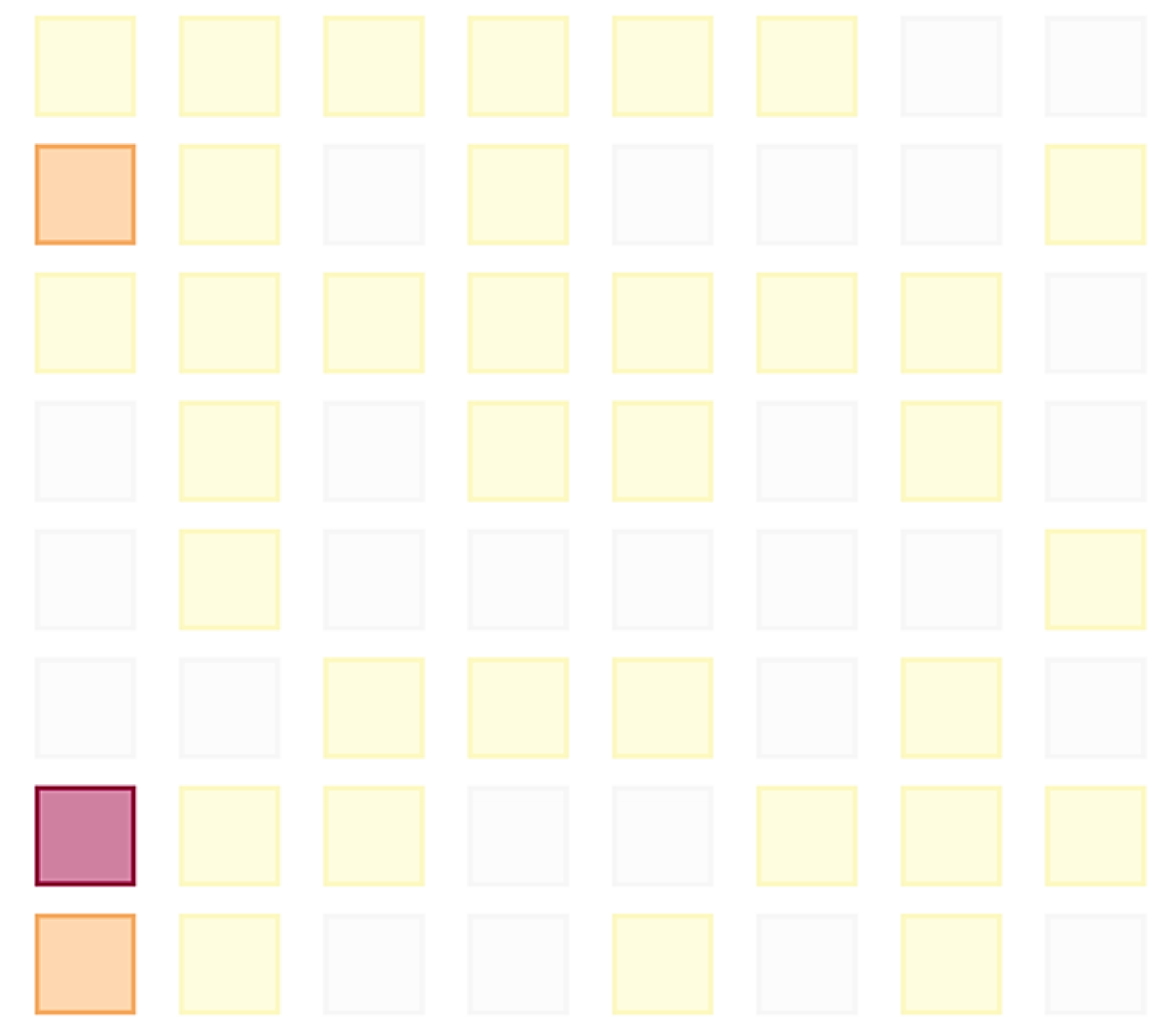}}\quad
  \subfloat[risk against the peers]{\includegraphics[scale=0.16]{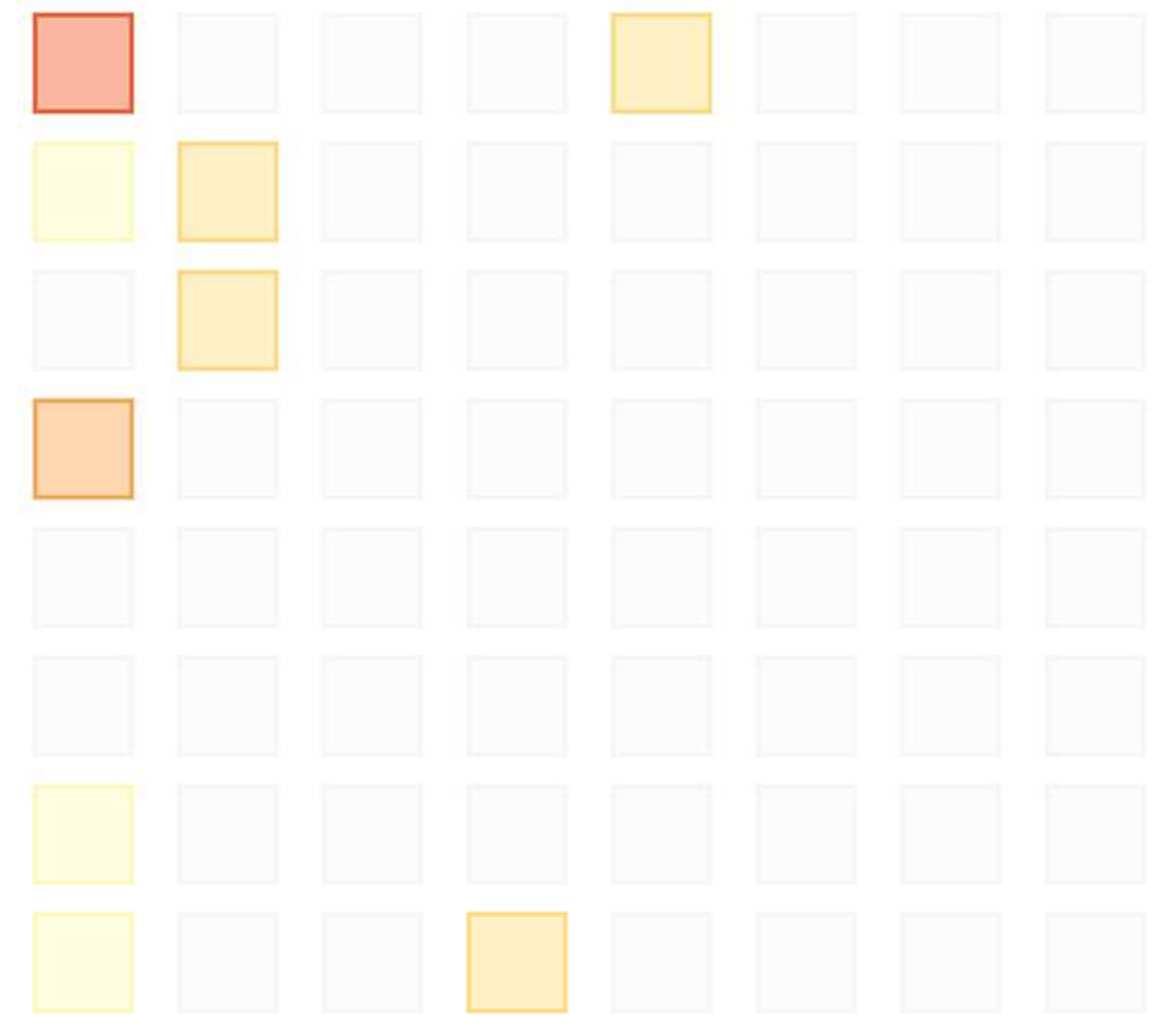}}
\caption{The topic grids. The self risk in (c) is derived from comparing the current activities (a) and the historical activities (b) of a specific entity. The peer risk in (e) is derived from comparing the current activities (a) and the peers' activities (d) of a specific entity.}
\label{fig-TG}
\end{figure*}

\begin{figure*}[ht]
\begin{center}
\centerline{\includegraphics[scale=1]{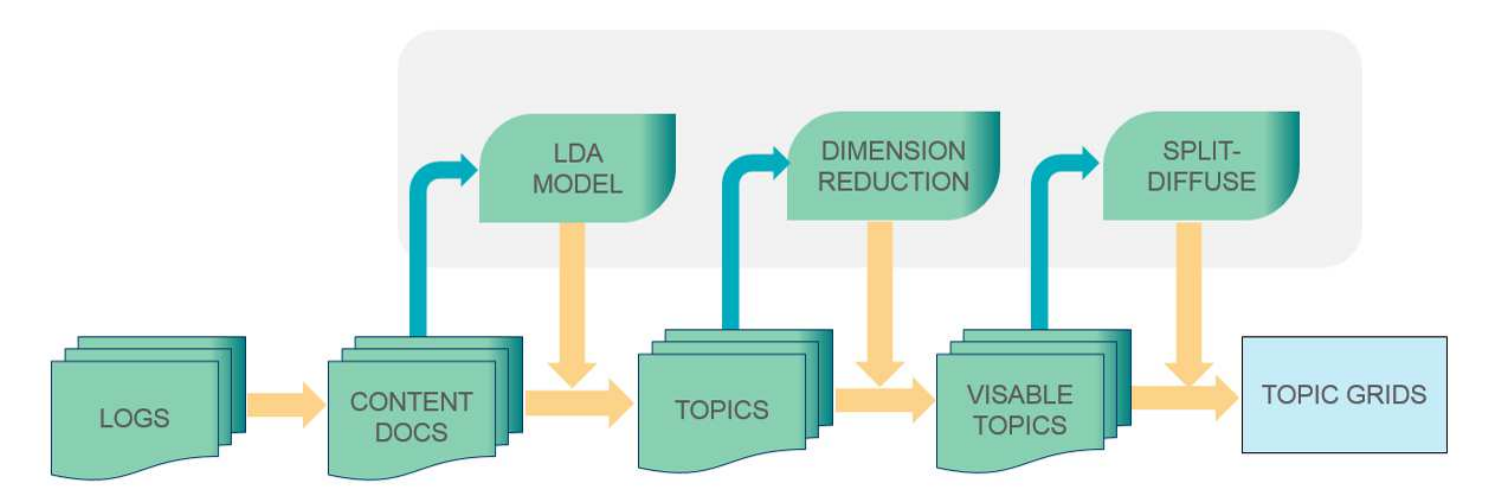}}
\caption{System blocks for generating the topic grids. The blue arrows are the training routes. There are three sub-systems to be trained by a benchmark data set: the topic model (e.g., LDA), the dimension reduction algorithm on the trained topics, and the SD algorithm on topics over the reduced dimension. Once these sub-systems are trained, the data manipulation in the shaded sub-systems is fixed and the blue routes are turned off until the next re-train. The online routes are in yellow.}
\label{fig-TGBlocks}
\end{center}
\end{figure*}

\section{The Split-Diffuse Algorithm}

\begin{algorithm}[tb]
   \caption{Split-diffuse algorithm ($k$-dimensional)}
   \label{alg-SD}
   {\bfseries Input:} data points $\{p\}$ of length $g_1\times g_2 \times \dots \times g_k$, allocation vector $\vec{\textbf{c}}=(0,\dots ,0) \in \left(\mathcal{Z}^+\right)^k$\\ 
   {\bfseries split-diffuse ($\{p\}$, $\vec{\textbf{g}}$, $\vec{\textbf{c}}$)}
\begin{algorithmic}
   \STATE $k \leftarrow $ length of $\{p\}$
   \IF{$k = 0$,} 
   \STATE return $null$
   \ELSIF{$k = 1$,} 
   \STATE resolve $\mathbb{S}(p)$ from $\vec{\textbf{c}}$, return $p$
   \ENDIF
   \STATE $a \leftarrow \argmax_a{(g_a)}$
   \STATE $w_l \leftarrow g_a$ integer divides $2$; 
   \quad $w_r \leftarrow g_a - w_l$
   \STATE $\vec{\textbf{g}_l} \leftarrow \vec{\textbf{g}}$ with $g_a$ being $w_l$;
   \quad $\vec{\textbf{g}_r} \leftarrow \vec{\textbf{g}}$ with $g_a$ being $w_r$
   \STATE $\vec{\textbf{c}_l} \leftarrow \vec{\textbf{c}}$;
   \quad $\vec{\textbf{c}_r} \leftarrow \vec{\textbf{c}}$ with $c_a$ being $c_a + w_r$
   \STATE $m \leftarrow w_l \prod_{i\neq a}g_i$
   \STATE $\{q\} \leftarrow \{p\}$ sorted in dimension $a$
   \STATE 
   \begin{tabular}{@{}ll@{}l} 
   return & ( & [{\bfseries split-diffuse} ($\{q_i:i\leq m\}$, $\vec{\textbf{g}_l}$, $\vec{\textbf{c}_l}$)], \\ 
   & & [{\bfseries split-diffuse} ($\{q_i:i> m\}$, $\vec{\textbf{g}_r}$, $\vec{\textbf{c}_r}$)])
   \end{tabular}
\end{algorithmic}
\end{algorithm}

We propose an algorithm called the split-diffuse (SD) algorithm to realize the strategy above. The idea of the SD algorithm is shown as a simple example in Figure~\ref{fig-SD}. There are $16$ points to be placed over a $4\times 4$ layout space $\mathbb{S}$. The SD algorithm first picks the $x$-dimension to split, and splits the data points into two groups, the ones smaller or equal to the median, and the ones not. Each group goes through this split step again over the $y$-dimension, as in Figure~\ref{fig-SD} (b). In this case, we recursively split the points in $x$- and $y$-dimension iteratively, until there is only one point in current recursion.

We keep track of the splitting path in the string $c$. At the end of the recursion, the placement of the single point $p$ in $\{p\}$ is resolved as $\mathbb{S}(p)$. The values of $\mathbb{S}(p)$ reflect the mapped position in the final placement and are integers in all dimension of $\mathbb{S}$. For example, the black dot $p'$ in Figure~\ref{fig-SD} (c) is placed at $\mathbb{S}(p')=(2,2)$, where the indexes begin with zero at the upper left corner. As a result, the mapped data points are equally spaced in the predefined layout of shape $g_1\times g_2 \times \dots \times g_k$. To achieve this uniformity in the space $\mathcal{L}$, the data points are essentially diffused from the denser area to the coarser area by the SD algorithm --- hence the name split-diffuse.

Some examples on the input and the output of the SD algorithm are shown in Figure~\ref{fig-mdstsne}. We generate $64$ topics regarding to the content of a repository access logs. The topics are mapped from $\mathcal{H}$ to $\mathcal{L}$ via MDS in Figure~\ref{fig-mdstsne} (a)(e) and t-SNE \cite{Van08} in Figure~\ref{fig-mdstsne} (c)(g). A topic is represented by the encrypted first three letters of the most descriptive word. Topics sharing the same representative word are the topics close to each other in $\mathcal{H}$.

In both of our MDS and t-SNE cases, the topics with the same representative word tend to form clusters. The corresponding output of the SD algorithm on MDS and t-SNE inputs are alsow shown in Figure~\ref{fig-mdstsne}. The clustering topology is maintained in a way when the uniformity of the point placement is enforced across all dimensions. The topical behavior can be further visualized on $\mathbb{S}$, as in Figure~\ref{fig-TG}. The behavioral metrics will be discussed later in Section~\ref{sec-ba}.

\begin{figure*}[htp]
  \centering
  \subfloat[mouse over a grid]{\includegraphics[scale=0.28]{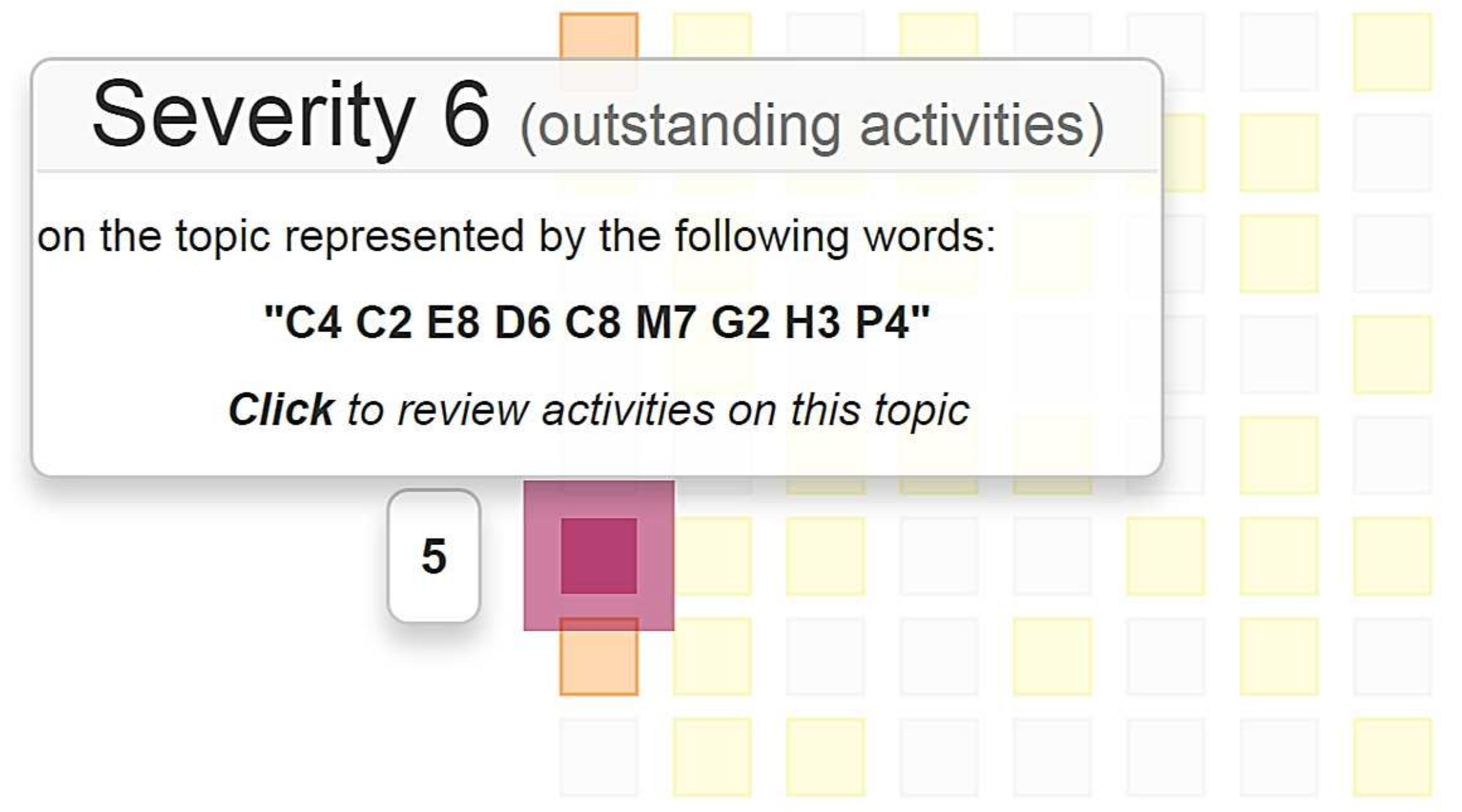}} \quad\quad\quad
  \subfloat[popup after clicking a grid]{\includegraphics[scale=0.40]{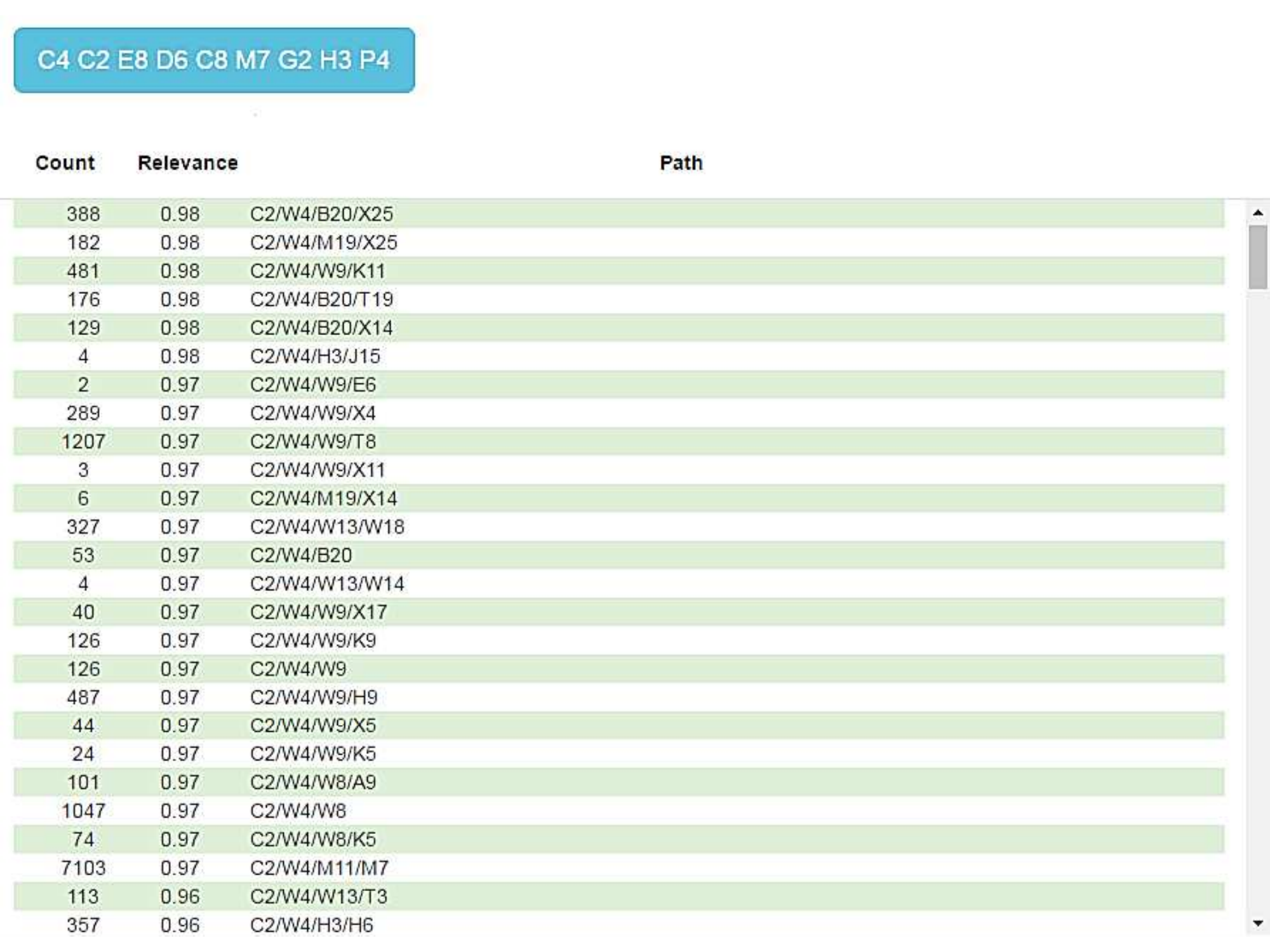}}
\caption{Interacting with the topic grids.}
\label{fig-Int}
\vskip -0.1in
\end{figure*}

\begin{figure*}[htp]
  \centering
  \subfloat[Topic curtain]{\includegraphics[scale=0.45]{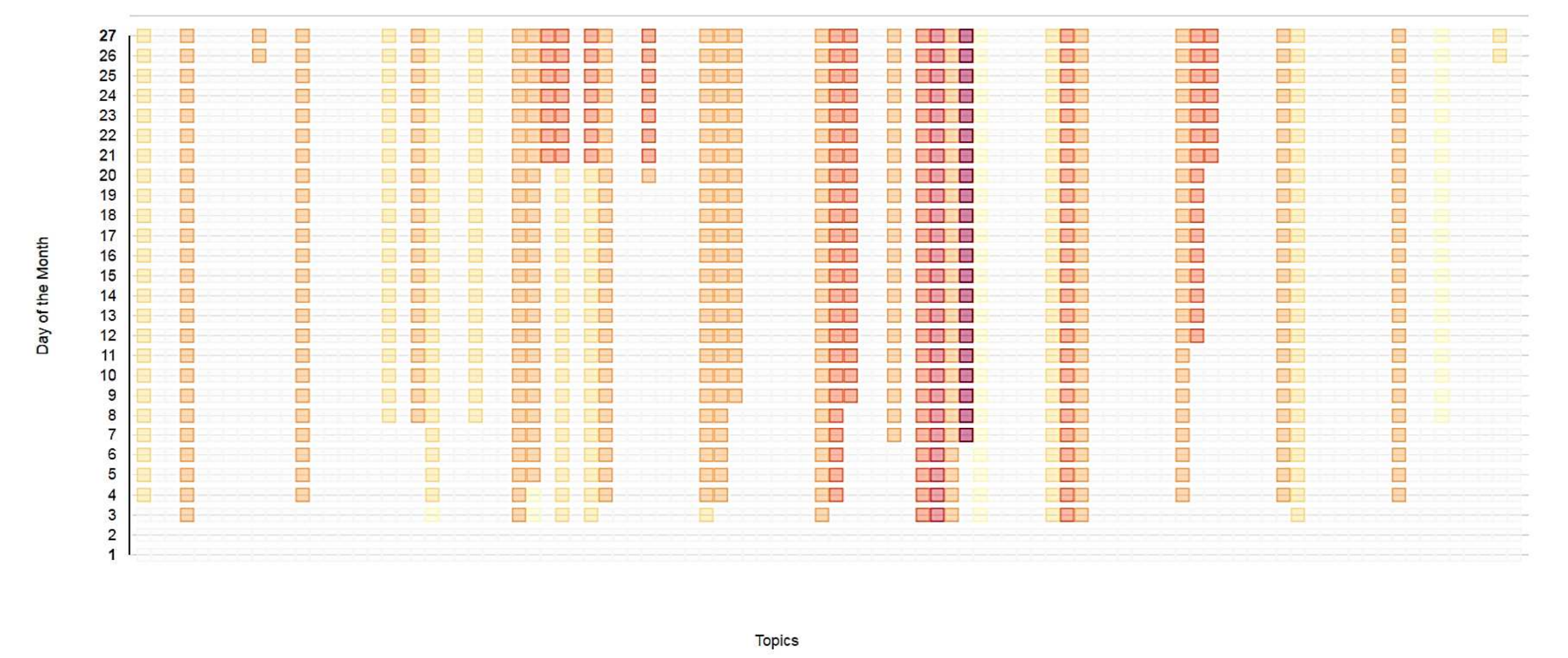}} 
  \subfloat[Topic shower]{\includegraphics[scale=0.21]{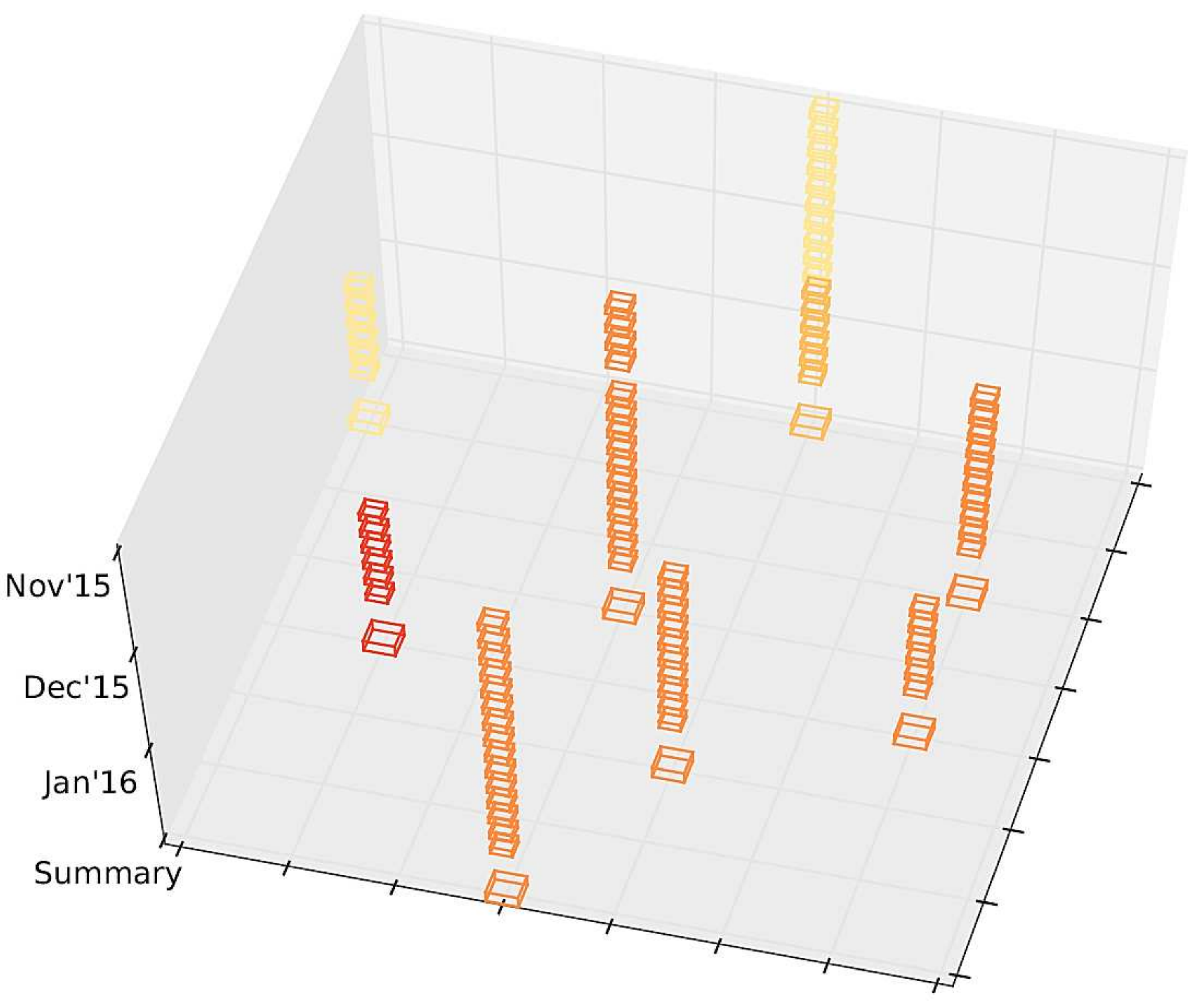}}
\caption{Other variants of the topic grids}
\label{fig-TCS}
\end{figure*}

\section{Behavior Anomaly and the Topic Grids}\label{sec-ba}

Consider the case that terabytes of network logs comes in everyday. Is there a way to quickly detect the anomaly entities and provide the visual breakdown about their behavioral change? In the cyber security domain, the SD algorithm is applied to analyze behavioral content via the proposed system in Figure~\ref{fig-TGBlocks}. It starts with the log files that record the user behavior of a repository. This data goes through and trains three sub-systems: the topic model, the dimension reduction, and the SD map. We use six month worth of data to ensure good content coverage fed into the topic model. Once trained, the topics, their coordinates in $\mathcal{L}$, and the SD map are then fixed for the incoming data. The system analyzes the risk over all topics for each user. Behavior anomaly is ranked and rendered to the human experts.

The goal of the system is to detect behavioral anomaly based on the access logs. After proper punctuation, the accessed path (or URL), the content, and/or any meta data regarding to an access can be viewed as a document, called the content document. We use the LDA model to decide the topics among all these content documents over the benchmark period of time. In one of our implementations, $64$ topics are generated at a word vector space of $19K+$ dimensions. The relevance between a content document and each individual topic is measured. The anomaly, or risk, of an access is quantified by the difference of topical relevance between this access and the benchmark profile of the entity. The topics are then projected to the 2D visualization space via any existing dimension reduction algorithm. We use MDS to maintain the global similarity. But it could be other algorithms emphasizing other point-wise relationships --- as long as the semantically similar topics are mapped close to each other.

The SD algorithm then places these $64$ topics over an $8\times8$ layout, forming a set of topic grids. Topical activities of an entity (user, group, node, etc), of the peers of the entity, of the history of the entity can all be compared on the same set of grids. Other topical metrics, such as risk or trend, can also be visualized over the same set of grids. One use case is presented in Figure~\ref{fig-TG}.

The activities of a specific user over a specific period of time is shown in Figure~\ref{fig-TG} (a), and is compared with the same metric of the same user over a historical benchmark period in Figure~\ref{fig-TG} (b). The color in Figure~\ref{fig-TG} (a) and (b) reflects the volume of the activities over the topic. Red means higher volume of topical activities. The difference between Figure~\ref{fig-TG} (a) and (b) is the anomaly against the historical behavior of the user, and is quantified in Figure~\ref{fig-TG} (c), where the red color indicates higher topical risk. One candidate topical risk metric between time periods $T_1$ and $T_2$ over topic $t$ for the same entity $e$ is
\begin{equation} \label{eq-risk}
R^{(B_{e,T_1},B_{e,T_2})}_t = \log(\sum_{a\in B_{e,T_2}}r_a+1) - \log(\sum_{a\in B_{e,T_1}}r_a+1),
\end{equation}
where $r_a$ is the relevancy for activity $a$ to topic $t$, and $B$ is the set of activities defined by the unique content documents of all activities logged within the corresponding time period. In our example, the red block at the upper left corner of Figure~\ref{fig-TG} (c) comes from the recent activities of the same topic (location) in Figure~\ref{fig-TG} (a), while having no historical activities on the same topic in Figure~\ref{fig-TG} (b). Similar comparison is made against the peers of the user as in Figure~\ref{fig-TG} (d) and (e). Changing the definition on $B$ or adding other quantifiers in Equation~\ref{eq-risk} leads to different activity and risk metrics, which can be visualized similarly as in Figure~\ref{fig-TG}.

When not directly displaying the detail keywords about a topic, the topic grids requires less space. At the same time, the human expert still can easily keep track of the topics based on their indexes over all dimensions and compare the difference between different sets of topic grids. Human interaction, which is the ultimate goal of the uniform placement of the data points, can be done more easily on the topic grids than on the raw dimension reduction output as in Figure~\ref{fig-mdstsne} (a) and (c). Example interactions shown in Figure~\ref{fig-Int} are the mouse over event to popup the topical summary, and the click event to overlay the detailed topical activities.

It is also possible to reserve one dimension of the visualization space as the time axis, to better analyze the behavior change over time. In this case, an one-dimensional SD algorithm can be applied to place the topics in the 1D space. The result is as in Figure~\ref{fig-TCS}(a), where x-axis represents the indexes of the topics and the y-axis is the time. In this implementation, we accumulate the behavior over the month. Therefore, the risk remain high for the rest of the month when a user touches a risky topic. This memory effect creates pattern like a curtain, thus the name topic curtain.  Meanwhile, we can pile up the 2D topic grids on the time axis over the 3D $\mathcal{L}$, as the topic shower shown in Figure~\ref{fig-TCS}(b). With normal or usual behavior, it is expected to see the consistent hot grids at the same locations over time.

\section{Performance of the Split-Diffuse Algorithm} 
 
In order to measure the performance of the split-diffuse algorithm, we randomly generate the topics in the 2D space and let the algorithm evenly space these topics over a predefined layout in the space of same dimensionality. The random topics are generated over a 2D space via two types of sampling approaches, namely the uniform approach,
\begin{equation} \label{eq-uni}
\begin{array}{ll}
\mathbb{U}(\rho) : \left[\begin{array}{c}X\\Y\end{array}\right]^T 
\left[ \begin{array}{cc}\rho&0\\0&1\end{array} \right],
&X,Y \sim \mathcal{U}(-0.5,0.5)
\end{array}
\end{equation}

and the Gaussian approach.
\begin{equation} \label{eq-gau}
\begin{array}{l}
\mathbb{G}(\theta,\phi) : \left[\begin{array}{c}X\\Y\end{array}\right]^T
\left[ \begin{array}{cc}\phi&0\\0&1\end{array} \right]
\left[ \begin{array}{cc}\cos\theta&\sin\theta\\-\sin\theta&\cos\theta\end{array} \right],
\\
\\
\left[ \begin{array}{c}X\\Y\end{array} \right] \sim \mathcal{N}\left(\left[ \begin{array}{c}0\\0\end{array} \right],I_{2}\right)
\end{array}
\end{equation}

As mentioned in Section~\ref{sec-background}, the best attempts will be made to preserve the the topology of these points in $\mathcal{L}$ before the $\mathbb{S}$ mapping. That is, when $p_i$ is to the left of $p_j$, $\mathbb{S}(p_i)$ should be to the left of $\mathbb{S}(p_j)$, for any dimension $l$ of $\mathcal{L}$. For the case of placing $n\triangleq\|\{p\}\|$ data points over the space $\mathcal{L}$ of $k$ dimensions, there are totally $k{n \choose 2}$ constraints to be met. The error on such topology-maintaining attempt is then measured as
\begin{equation} \label{eq-I}
Err_{I} \triangleq \frac{e_{I,1}+e_{I,2}}{ {n \choose 2} \cdot k}
\end{equation}
where
\begin{equation*}
\begin{array}{l}
e_{I,1} \triangleq \|\{p_i<p_j, \mathbb{S}(p_i)\geq \mathbb{S}(p_j)\}\|_l \\
e_{I,2} \triangleq \|\{p_i\geq p_j, \mathbb{S}(p_i)<\mathbb{S}(p_j)\}\|_l
\end{array}
\end{equation*}
for all $i,j$ and for all $l^{th}$-dimension of $\mathcal{L}$.

One can generally regard the mapped value $\mathbb{S}(p)$ to be integer. There will be some amount of points sharing the same $\mathbb{S}(p)$ value in each of the dimension. Therefore, a loosen version of the error metric is: when $p_i$ is to the left of $p_j$, $\mathbb{S}(p_i)$ should not be to the right of $\mathbb{S}(p_j)$, for any dimension $l$ of $\mathcal{L}$. This loosen metric can be defined as

\begin{equation} 
Err_{II} \equiv \frac{e_{II,1}+e_{II,2}}{ {n \choose 2} \cdot k}
\end{equation}
where
\begin{equation*} 
\begin{array}{l}
e_{II,1} \equiv \|\{p_i<p_j, \mathbb{S}(p_i)>\mathbb{S}(p_j)\}\|_l \\
e_{II,2} \equiv \|\{p_i>p_j, \mathbb{S}(p_i)<\mathbb{S}(p_j)\}\|_l \\
\end{array}
\end{equation*}
for all $i,j$ and for all $l^{th}$-dimension of $\mathcal{L}$.

\subsection{Errors on Grid Layouts}

The error ratios on split-diffusing $\mathbb{U}(1)$ and $\mathbb{G}(\frac{4}{\pi},2)$ over various grid layouts are shown in Table~\ref{table-mer}. $Err_{I}$ decreases as the grid set size increases on both $\mathbb{U}(1)$ and $\mathbb{G}(\frac{4}{\pi},2)$ sampling schemes. However, $Err_{II}$ only decreases as the grid set size increases on the $\mathbb{U}(1)$ sampling. $Err_{II}$ increases as the grid set size increases on the $\mathbb{G}(\frac{4}{\pi},2)$ sampling.

\begin{table*}[t]
\caption{Mean error ratio over various random sampling schemes and various layouts in the 2D space. Both $Err_{I}$ and $Err_{II}$ are averaged over 1000 sets of samples from the indicated sampling schemes.}
\label{table-mer}
\begin{center}
\begin{small}
\begin{sc}
\begin{tabular}{ccccc}
\hline
Layout & Sampling & Constraints & $Err_{I}$ & $Err_{II}$ \\
\hline
$4\times4$ & $\mathbb{U}(1)$   & 240   & 0.2042 & 0.0292 \\
$8\times8$ & $\mathbb{U}(1)$   & 4,032  & 0.1347 & 0.0270 \\
$16\times16$ & $\mathbb{U}(1)$ & 65,280 & 0.0776 & 0.0192 \\
$32\times32$ & $\mathbb{U}(1)$ & 1,047,552 & 0.0426 & 0.0124 \\
$64\times64$ & $\mathbb{U}(1)$ & 16,773,120 & 0.0228 & 0.0074 \\
\hline
$4\times4$ & $\mathbb{G}(\frac{4}{\pi},2)$   & 240   & 0.2368 & 0.0618 \\
$8\times8$ & $\mathbb{G}(\frac{4}{\pi},2)$   & 4,032  & 0.1845 & 0.0769 \\
$16\times16$ & $\mathbb{G}(\frac{4}{\pi},2)$ & 65,280 & 0.1459 & 0.0875 \\
$32\times32$ & $\mathbb{G}(\frac{4}{\pi},2)$ & 1,047,552 & 0.1242 & 0.0940 \\
$64\times64$ & $\mathbb{G}(\frac{4}{\pi},2)$ & 16,773,120 & 0.1131 & 0.0977 \\
\hline
\end{tabular}
\end{sc}
\end{small}
\end{center}
\end{table*}

Another observation is that although $Err_{I}$ decreases as the grid set size increases on both $\mathbb{U}(1)$ and $\mathbb{G}(\frac{4}{\pi},2)$ samplings, the decrease rate varies. On the $\mathbb{U}(1)$ sampling, $Err_{I}$ decreases almost $90\%$ when the layout grows from $4\times4$ to $64\times64$. However, $Err_{I}$ only decreases around half of its value during the same layout growth for the $\mathbb{G}(\frac{4}{\pi},2)$ sampling. By definition, $Err_{I} \geq Err_{II}$ on any given layout-sampling combination. Therefore, over the $\mathbb{G}(\frac{4}{\pi},2)$ sampling on Table~\ref{table-mer}, the decreasing of $Err_{I}$ and the increasing of $Err_{II}$ will not incur a cross-over as the layout becomes larger.

\subsection{Errors on Distribution of Input Data Points}

The error ratios on split-diffusing various sampling schemes over the same $8\times8$ grid layout are shown in Figure~\ref{fig-err2}. The error ratio reflects the portion of the constraints in the specific dimension that are not met. We also ask the SD algorithm to start the splitting in the $y$-dimension.

\begin{figure*}[htp]
  \centering
  \subfloat[$\mathbb{U}(\rho)$ samples]{\includegraphics[scale=0.15]{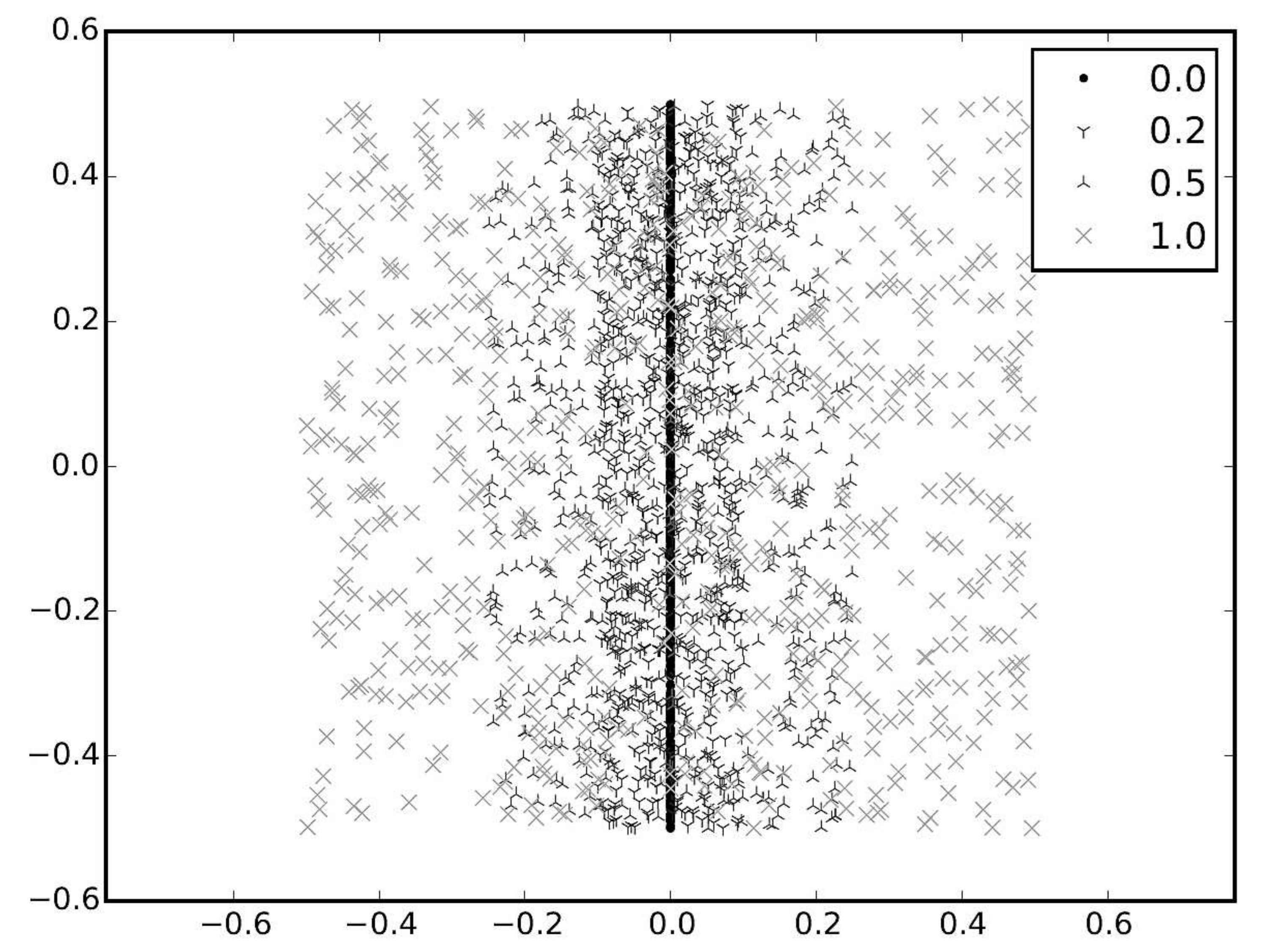}}
  \subfloat[$\mathbb{G}(\theta,2)$ samples]{\includegraphics[scale=0.15]{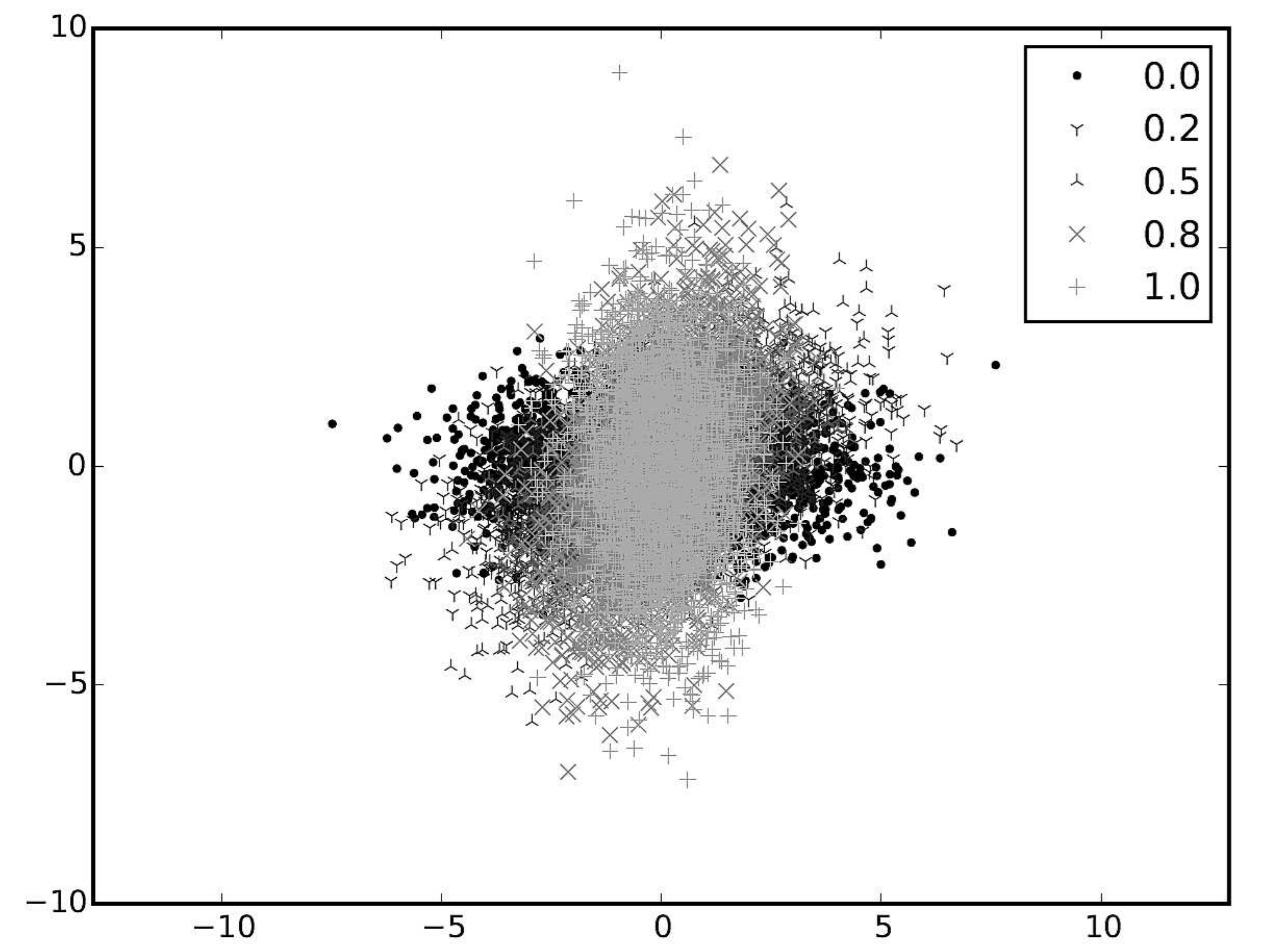}}
  \subfloat[$\mathbb{G}(\theta,5)$ samples]{\includegraphics[scale=0.15]{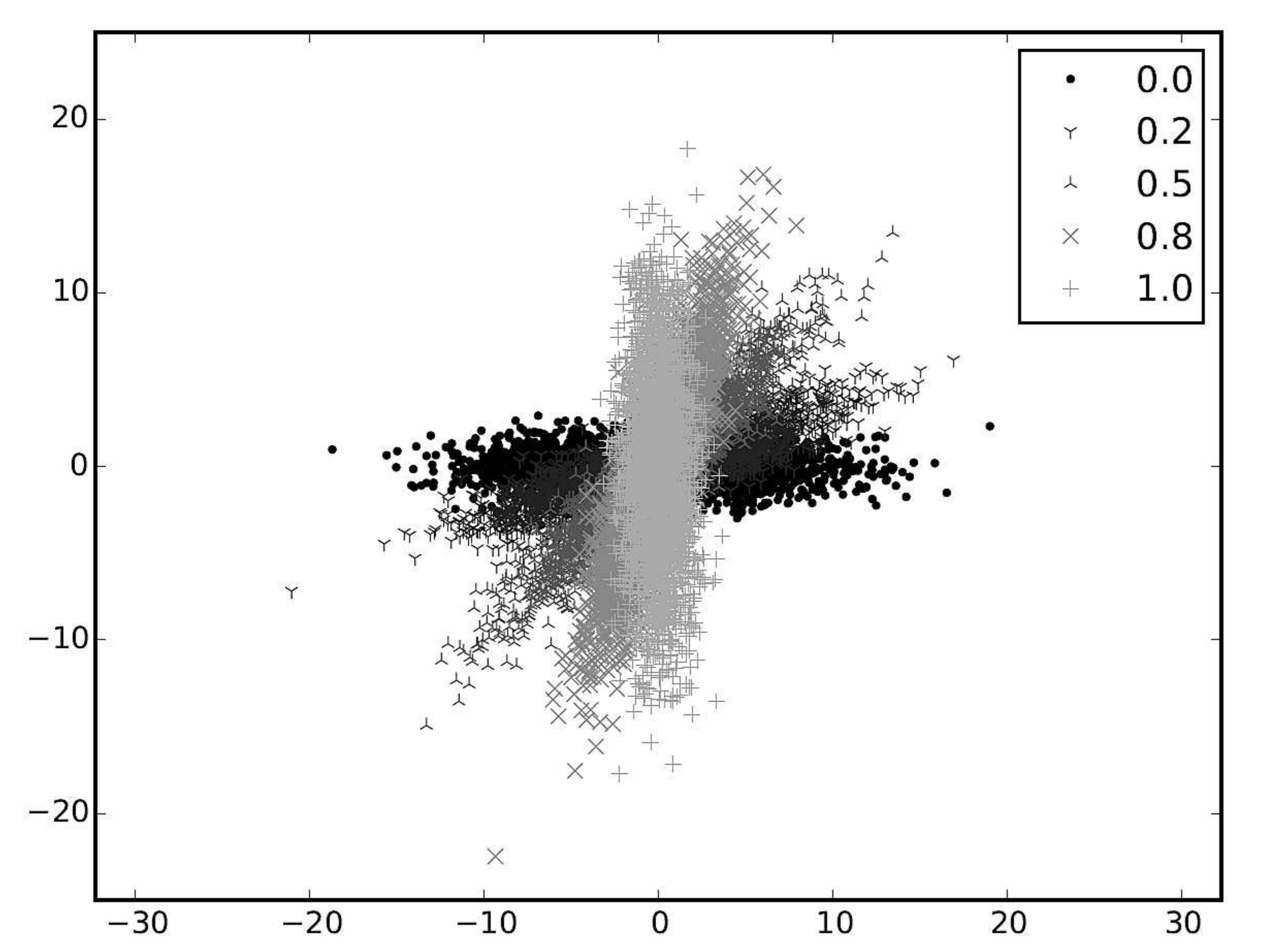}}
  \subfloat[$\mathbb{G}(\frac{\pi}{2},\phi)$ samples]{\includegraphics[scale=0.15]{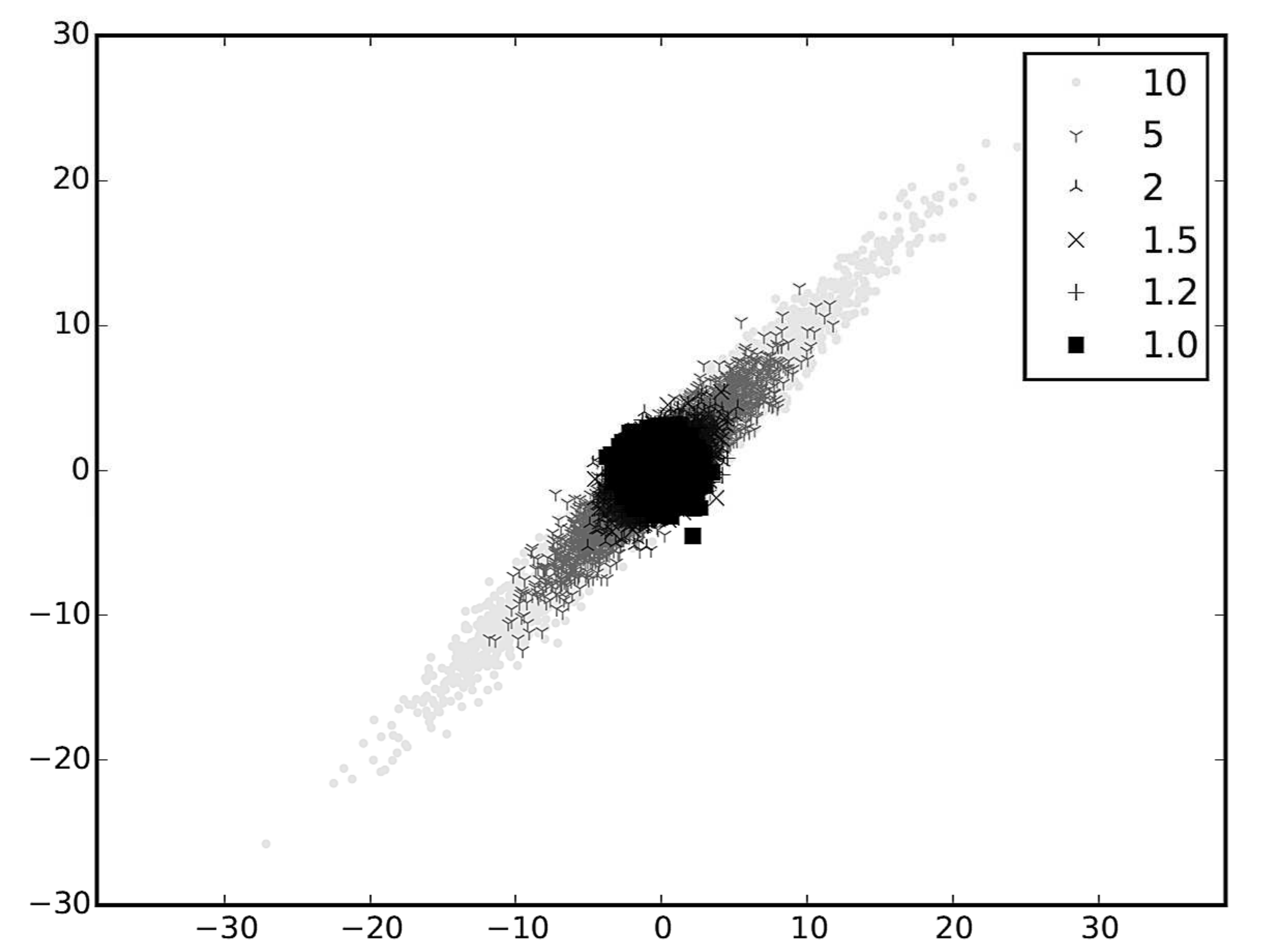}}
  \\
  \subfloat[$\mathbb{U}(\rho)$ errors]{\includegraphics[scale=0.15]{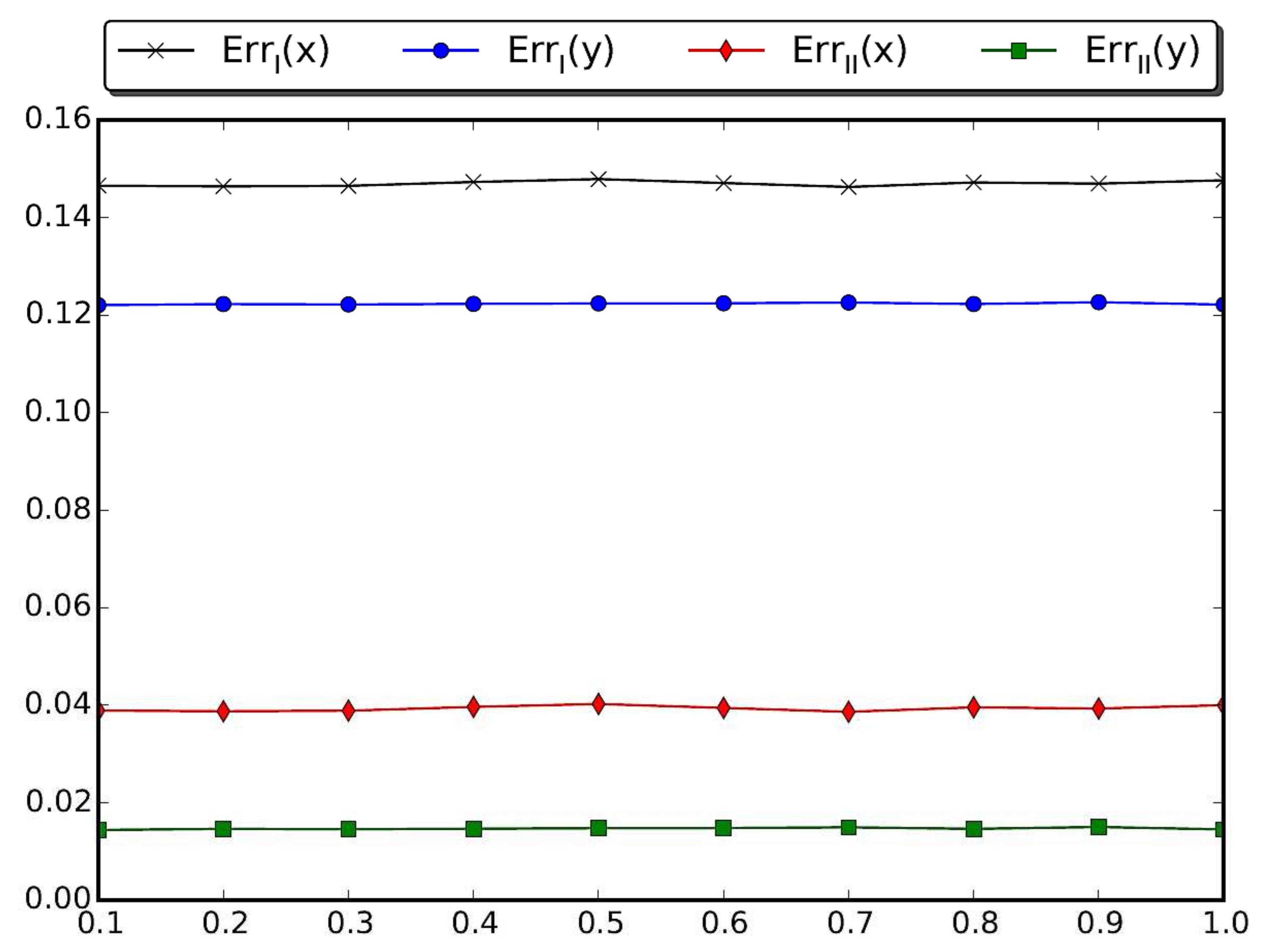}}
  \subfloat[$\mathbb{G}(\theta,2)$ errors]{\includegraphics[scale=0.15]{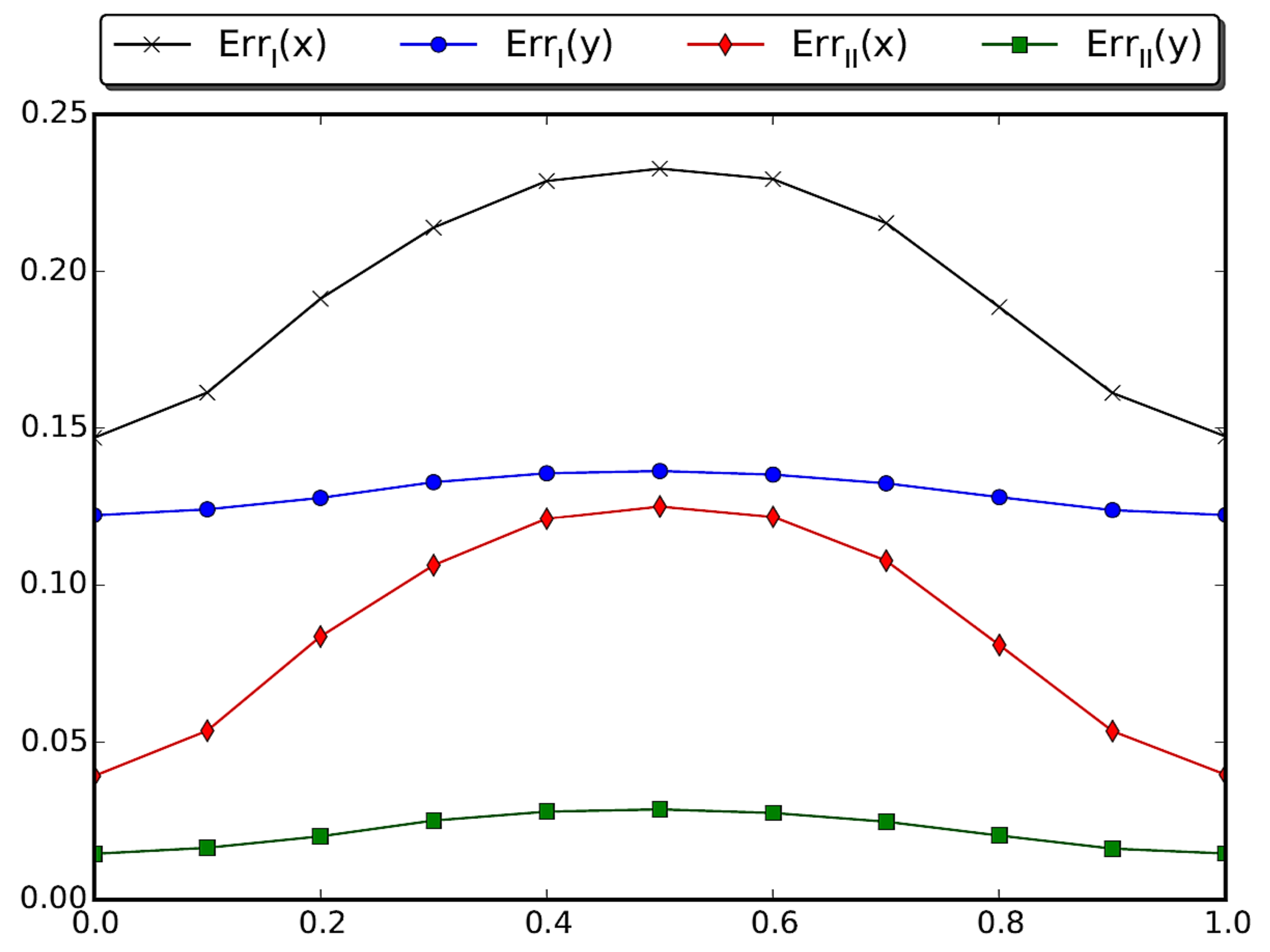}}
  \subfloat[$\mathbb{G}(\theta,5)$ errors]{\includegraphics[scale=0.15]{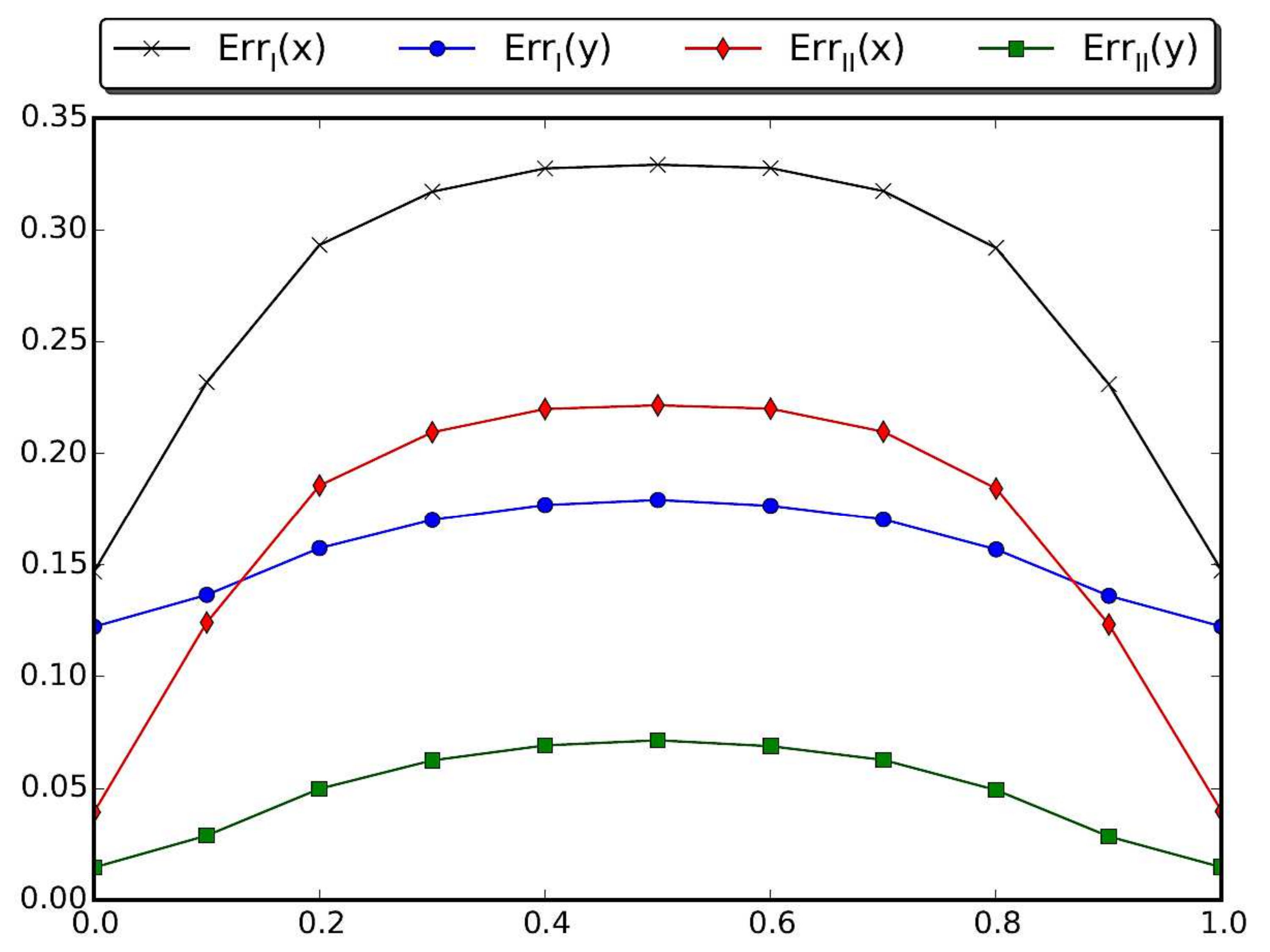}}
  \subfloat[$\mathbb{G}(\frac{\pi}{2},\phi)$ errors]{\includegraphics[scale=0.15]{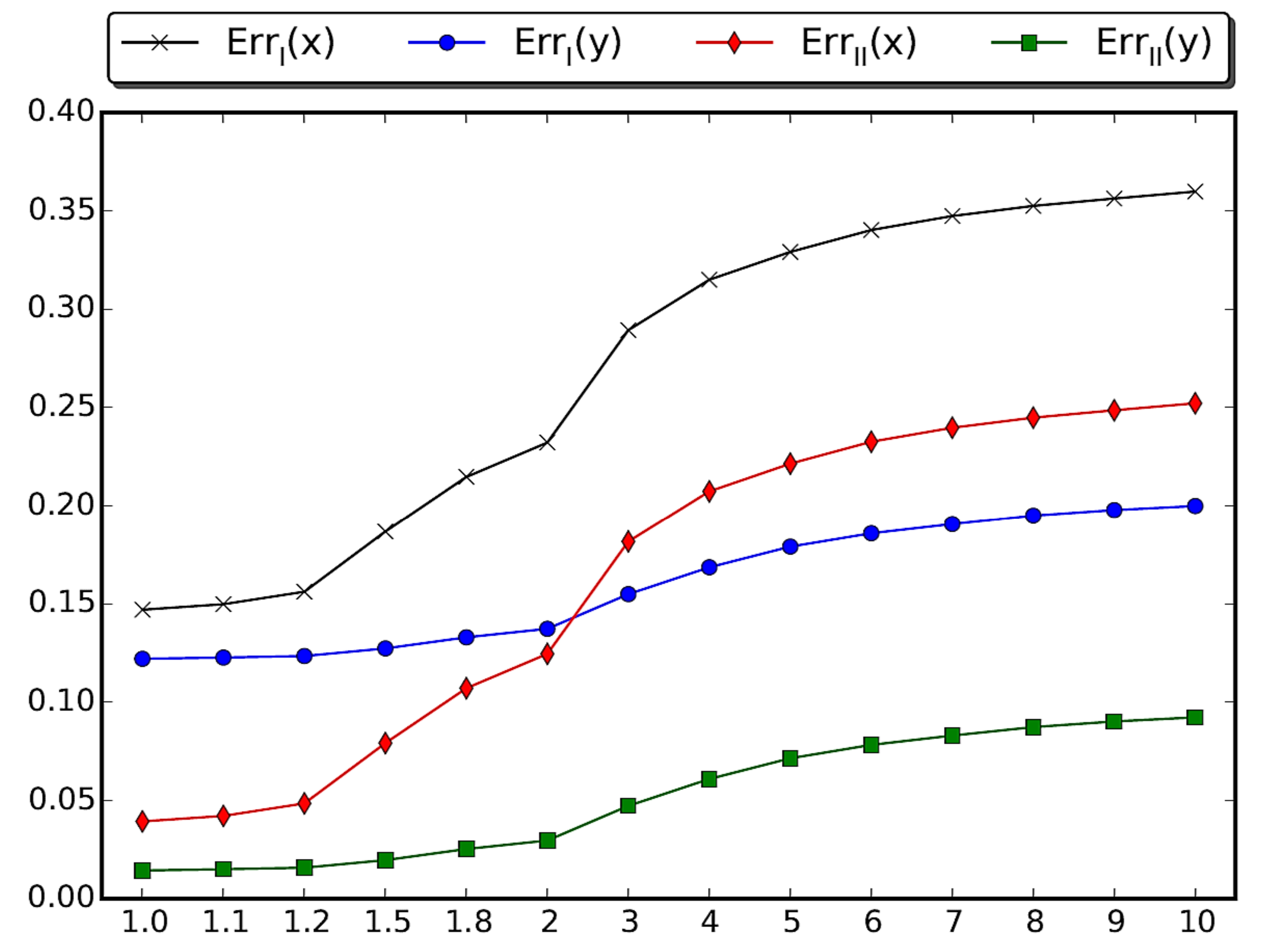}}
\caption{Error ratios on various sampling schemes over an $8\times8$ layout. Both $Err_{I}$ and $Err_{II}$ are averaged over 1000 sets of layout samples in (e)-(h), and the first 50 sets are scatter-plotted in (a)-(d). $\theta$ is plotted as multiples of $\pi/2$ in (b)(c)(f)(g).}
\label{fig-err2}
\end{figure*}

On the test over $\mathbb{U}(\rho)$, the sampled points uniformly distributed within a unit square are being shrunk on their $x$-values by the ratio $\rho$ (Equation~\ref{eq-uni}). When the SD algorithm starts the splitting in the $y$-dimension first, $Err_{I}(y)<Err_{I}(x)$ and $Err_{II}(y)<Err_{II}(x)$, as in Figure~\ref{fig-err2} (e). $\rho$ does not impact on all error metrics. If opting to start the SD algorithm in the $x$-dimension first, the resulting error ratios in the $x$- and the $y$-dimension would be swapped. (Plot is omitted)

In Figure~\ref{fig-err2} (f), we evaluate the error ratios of $\mathbb{G}(\theta,2)$. Under this sampling scheme, the error ratios vary with $\theta$. With a fixed $\phi$, both $Err_{I}$ and $Err_{II}$ have maximum values at $\theta=\pi/4$. However, the error ratio in the $x$-dimension (or the dimension not being chosen to start the split) is more sensitive to $\theta$ than that in the $y$-dimension, for both types of error. In comparison, when $\phi$ is fixed at a larger value, the Gaussian sample distribution will be narrower in shape. The narrower shape incurs higher error ratios, especially when $\theta$ is around $\pi/4$. 

Based on Figure~\ref{fig-err2} (f)(g), we find that Gaussian samples from Equation~\ref{eq-gau} with larger $\phi$ and $\theta$ around $\pi/4$ incur higher error ratios. To observe further in this type of Gaussian samples, we plotted $\mathbb{G}(\frac{\pi}{2},\phi)$ in Figures~\ref{fig-err2} (h). As expected, all error ratios increase as $\phi$ increases, where the error ratios on the $x$-dimension are more sensitive to $\phi$. The errors converge as $\phi$ becomes larger.

\begin{figure*}[htp]
  \centering
  \subfloat[splitting strategies]{\includegraphics[scale=0.20]{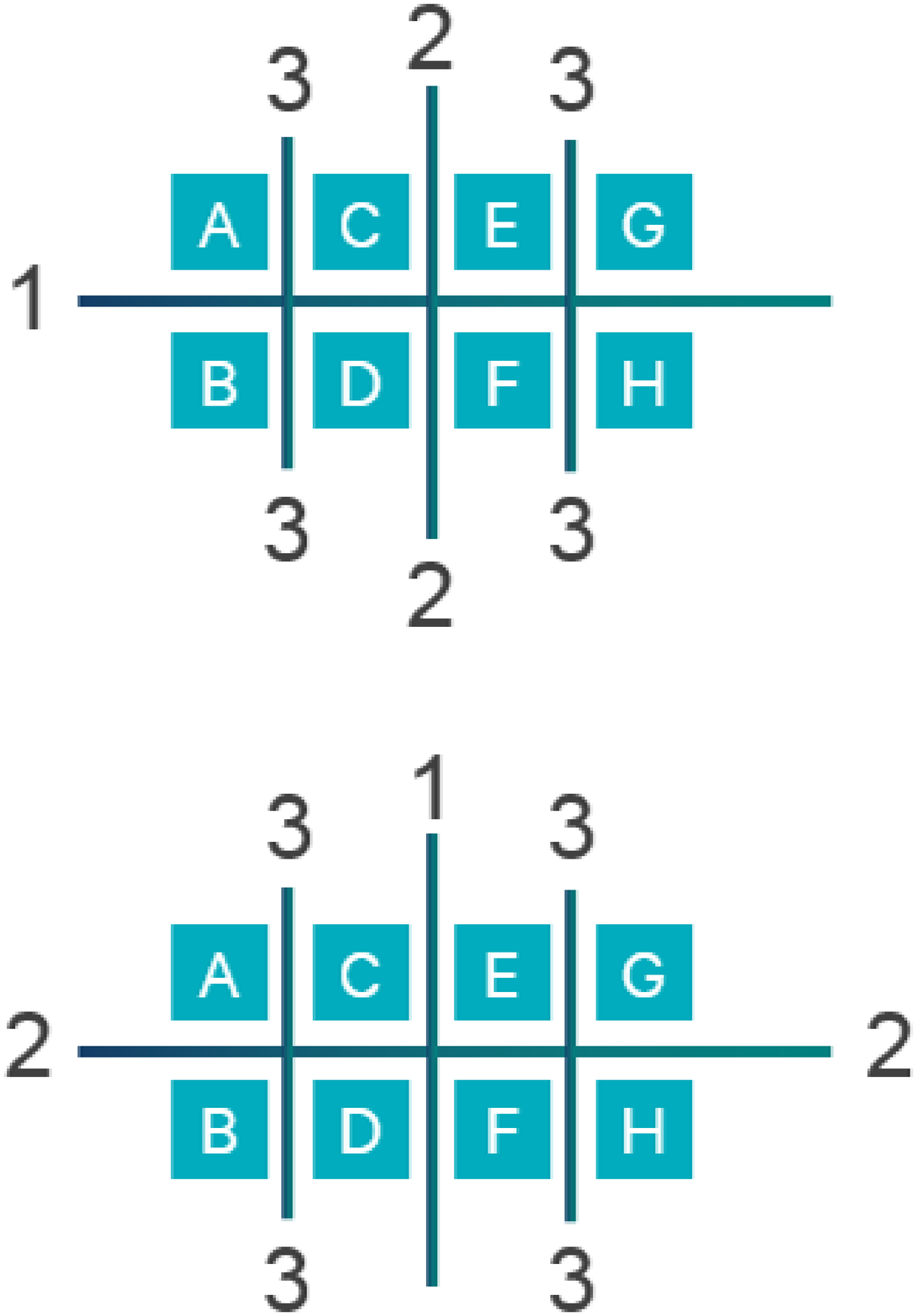}}\quad 
  \subfloat[errors on first splits]{\includegraphics[scale=0.20]{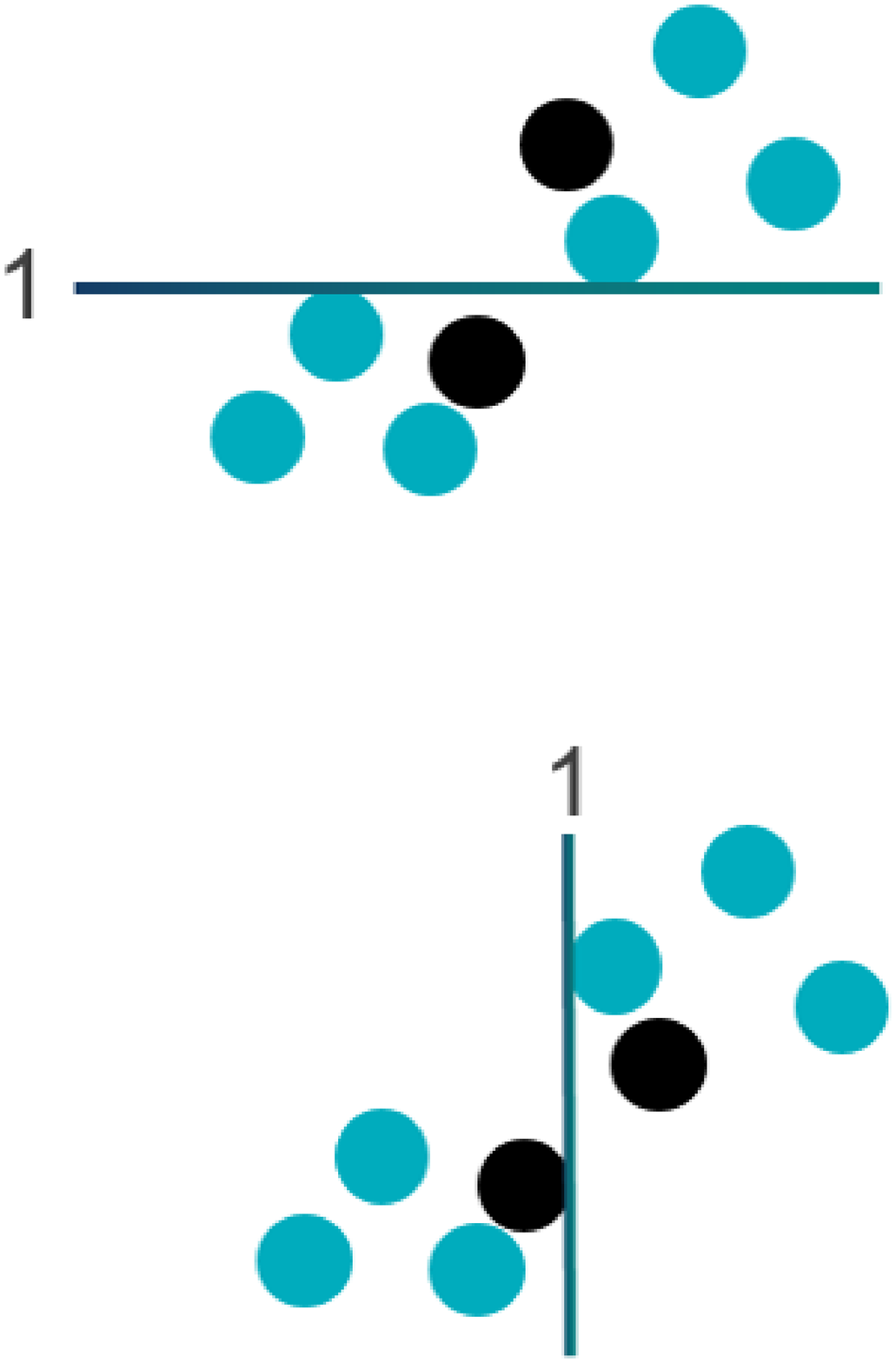}}\quad 
  \subfloat[$\mathbb{G}(\theta,2)$ errors]{\includegraphics[scale=0.24]{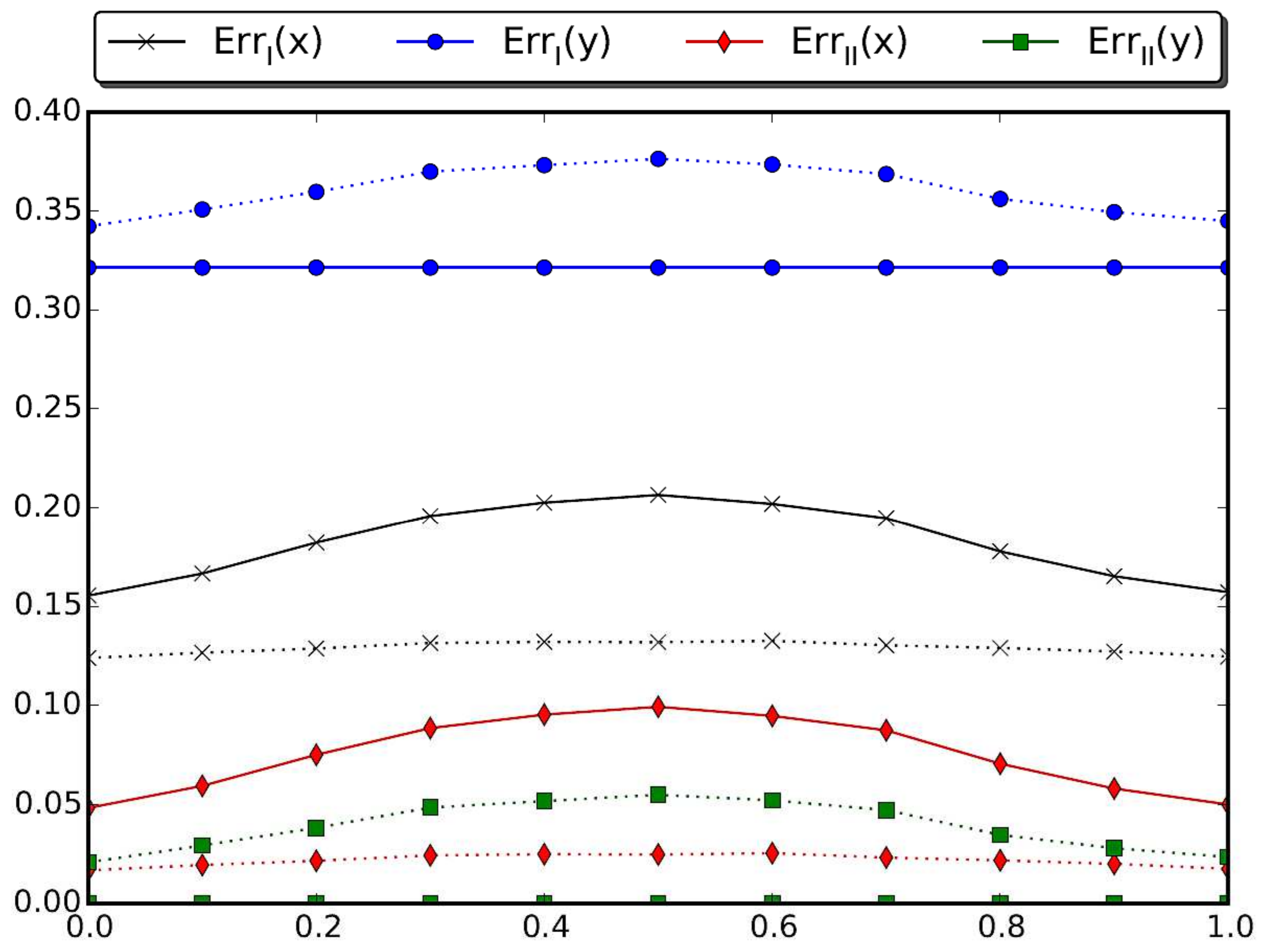}}
\caption{The split-diffuse algorithm over a $2\times4$ layout. The upper part of (a) shows an iterative splitting strategy that starts in the $y$-dimension; the lower part of (a) adopts the greedy splitting of the generalized SD algorithm. The errors incurred by the first split on both cases over the same set of input points are shown in (b): the black points are the example pairs where the errors occur. The error ratios over $1000$ $\mathbb{G}(\theta,2)$ input sample sets are plotted in (c) with the iterative splitting using the solid lines. $\theta$ is plotted as multiples of $\pi$.}
\label{fig-SD42}
\end{figure*}

\subsection{Error Bounds}\label{subsec-eb}

As mentioned above, each pair of points has one constraint on each dimension within the space. In 2D space, there are totally $n(n-1)$ constraints, where $n$ is the number of points in $\{p\}$. We start from inspecting the constraints toward one single point. In Figure~\ref{fig-SD} (c), the split path of a particular (black) data point is illustrated over a $4\times4$ layout. The split $\#1$ ensures that the 8 blue points to the left of the split $\#1$ line are mapped to the left of the mapped black point. So $8$ constraints in the $x$-dimension is met. Similarly, the split $\#2$ line guarantees $4$ constraints in the $y$-dimension is met, and so on. Upon the final split relating to the black point, that is when it is separated from all other points, a total of $15$ constraint on the black point will be met. The same applies to all other points.

Therefore, the SD algorithm satisfies ${n \choose 2}$ constraints in the 2D space. In other words, $Err_{I}$ is bounded by $1/2$. As a particular example, if the input data contains $4$ points along the diagonal, the SD algorithm meets $7$ of the $12$ constraints over the $2\times2$ layout ($3$ in $x$-dimension and $4$ in $y$-dimension if we split in the $y$-dimension first). A more general bound is as follows.
\begin{theorem}\label{thm-err}
In a space $\mathcal{L}$ of $k$ dimensions, the SD algorithm satisfies $1/k$ of the type I constraints measured by $Err_{I}$ in Equation~\ref{eq-I}. That is,
\begin{equation} 
Err_{I}\leq\frac{k-1}{k}
\end{equation}

\end{theorem}
 
\begin{proof}[Proof]
At the end of the SD algorithm, each data point $p$ is split into a cell that only contains $p$. The splits toward the point $p$ satisfy at least $(n-1)$ type I constraints involving $p$, across all dimensions of $\mathcal{L}$. Among all the points in $\{p\}$, a total of ${n \choose 2}$ type I constraints out of all $\left[{n \choose 2} \cdot k\right]$ type I constraints over all dimensions of $\mathcal{L}$ are met.
\end{proof}

On the trivial case of $k=1$, the data points will be evenly placed onto a one-dimensional space. There are ${n \choose 2}$ type I constraints, and all of them are met. The error bound become looser with larger $k$, which can be seen as another curse of dimensionality. Fortunately, we are only interested in visualizing the data points in a space having $k\leq 3$.

\subsection{Error on Non-square Layout}

On the general SD algorithm, one could argue on the dimension to start the first split. Consider the $2\times4$ layout in Figure~\ref{fig-SD42} (a). The upper part of Figure~\ref{fig-SD42} (a) illustrates the strategy to split iteratively over $x$- and $y$-dimensions while starting with the split in the $y$-dimension as default. The lower part of Figure~\ref{fig-SD42} (a) adopts greedy fashion of the SD algorithm with $\argmax_a{(g_a)}$ breaks tie in favor of the $y$-dimension.

The error bound in Theorem~\ref{thm-err} holds on both splitting strategies, since both of them ensure the splitting to end with every point being placed into a single-point cell. We exercise both strategies in Figure~\ref{fig-SD42} (a) on an example set of $8$ data points as shown in Figure~\ref{fig-SD42} (b). Inspecting the errors incurred by the first split, the iterative strategy in the upper part of Figure~\ref{fig-SD42} (b) guarantees the $y$-dimension constraints (Equation~\ref{eq-I}) between the group above split $\#1$ and the group below will be met. However, as the consequence of this split, the black points are eventually placed into the ``A'' and ``H'' positions. This placement violates the $x$-dimension constraint by four grids. The first split of the greedy splitting in the lower Figure~\ref{fig-SD42} (b), on the other hand, leads the black points into the ``C'' and ``F'' positions. The consequence $y$-dimension constraint violation has a displacement of two grids. Although both above cases count for one error in $Err_{I}$, the chance of incurring error with larger displacement makes the iterative splitting strategy less favored than the greedy splitting strategy of Algorithm~\ref{alg-SD}.

The error ratio comparison between both strategies over $\mathbb{G}(\theta,2)$ input samples are plotted in Figure~\ref{fig-SD42} (c). The overall $\mathbb{G}(\pi/4,2)$ $Err_{I}$ for the iterative splitting strategy is $0.2640$, compared to $0.2541$ from the greedy strategy. $Err_{II}$ for the these two strategies over the same setup are $0.0497$ and $0.0399$, correspondingly. From Figure~\ref{fig-SD42} (c) one can tell that the area between the dotted blue line and the solid blue line is smaller than the area between the dotted and solid black lines. The difference between these two areas reflects the $Err_{I}$ gain/improvement for choosing the greedy splitting over the iterative splitting. Without considering the amount of displacement on the error, the greedy splitting shows better $Err_{I}$ than the iterative splitting, with the maximum gain of $0.0098$ happening at $\theta=\pi/4$ of $\mathbb{G}(\theta,2)$. Similarly, greedy splitting also provides better $Err_{II}$ across all $\theta$.

Therefore, via actively looking for the dimension that has the most allocation spots to split, the greedy splitting strategy avoids having larger displacement. In addition, the better error ratios also led us to use this strategy in Algorithm~\ref{alg-SD}. There could possibly exist a better strategy toward a specific displacement metric, to be explored by further studies.

\subsection{Relationship to the $k$-d Tree}

The SD algorithm is similar to the generation algorithm of the $k$-d tree\cite{Bentley75} with some major differences
\begin{enumerate}
  \item The $k$-d tree creates a tree; the SD algorithm creates a map.
  \item The $k$-d tree splits in each dimension in turns; not always true for the SD algorithm.
  \item The $k$-d tree stores the median point at the split node; the SD algorithm divides points in two groups at each split.
  \item The $k$-d tree is used for fast similarity search; the SD algorithm places points evenly-spaced.
  \item The $k$-d tree is evaluated by the scale and speed of search; we evaluate the SD algorithm, or similar algorithms that map data points uniformly, by the number of constraints it satisfies.
\end{enumerate}

While $k$-d tree can be used on high dimensional data points and suffers from the curse of dimensionality, the SD algorithm is usually performed on the 2D or 3D space. The complexity of the SD algorithm is comparable to that of the $k$-d tree. When using the $\mathcal{O}(n)$ median of medians algorithm \cite{Blum73} in selecting the median at each depth level, the worst-case complexity is $\mathcal{O}(n\log n)$.

\section{Conclusion and Future Work}

In this paper, we introduce the topic grids to organize and visualize massive amount of behavioral data. Once trained, the same set of topic grids can be used to visualize different metrics on different targets over different periods of time, and compare the difference. Human can perceive the difference easily via the location of the integer-indexed grids. Topical details can be further rendered when the user interacts with the topic grids. Although we use the squared layout in the paper, the technique can produce layouts in rectangles, cubes or cuboids. The goal behind the technique is to utilize the visualization and interaction space as much as we can, for making topical comparison.

We present the use case to analyze the behavioral risk over the network activities. Beyond risk, the technique can be used to analyze topical trend, expertise, learning rates, or other metrics of interest. It can be applied to other domains beyond cyber security. Working with the e-commerce logs, the topic grids can be used to analyze the shopping behavior and preference of the customers, or the inventory of the vendor, over peer groups defined on location, age, or other attributes. Applications can further include credit card transactions, customer complaints, reviews of the shared economy companies, or any data source having free-form text from which we can extract topics.

\subsection{Conditioned Topic Placement}

The SD algorithm delegates the maintenance of the $\mathcal{H}$-to-$\mathcal{L}$ point-wise relationship to the existing word embedding algorithms of choice. There could be a possible way to map and maintain the $\mathcal{H}$-to-$\mathcal{L}$ point-wise relationship while enforcing the uniform spacing in $\mathcal{L}$. Such algorithm, when available, is a more direct algorithm to embed words uniformly. The desired aspect can be measured as displacement in the uniform $\mathcal{L}$ space, while the algorithm optimizes this metric. On a specific situation, it is desirable to keep the clusters in $\mathcal{H}$ intact after mapped to the uniform $\mathcal{L}$ space. Currently the SD algorithm cannot guarantee such property.

\subsection{Topic Drift over Time}

Over time, the topics on the input corpus may drift. There could be new topics while old topics fade away. On the assumption that most of the topics are still the same and having similar relationship to each other, it will be desirable to generate a new SD map $\mathbb{S}'(p)$ which looks similar to the old $\mathbb{S}(p)$.

\subsection{Application to Structured Data}

It is also possible to apply the topic grids to the structured data, on which an arbitrary clustering algorithm can generate cluster centers. The data points are then organized into these cluster centers, the same way we use the topic to represent the log entries related to it.

\bibliographystyle{IEEEtran}
\bibliography{IEEEabrv,jsu2016}

\end{document}